%% file: arXiv.tex
\documentclass[10pt,twocolumn,letterpaper]{article}

\usepackage[pagenumbers]{cvpr} % To force page numbers, e.g. for an arXiv version

\usepackage[toc,page]{appendix}
\usepackage{tikz}
\usepackage{pgfplots}
\usepackage{multirow}

\usepackage{graphicx}
\usepackage{amsmath}
\usepackage{amsthm}
\usepackage{amssymb}
\usepackage{booktabs}
\usepackage{mathtools}
\usepackage{xcolor}
\usepackage{comment}
\usepackage{tikz}
\usepackage{pgfplots}
\usepackage[normalem]{ulem}
\usetikzlibrary{decorations.pathreplacing}

\usepackage[pagebackref,breaklinks,colorlinks]{hyperref}

\usepackage[capitalize]{cleveref}
\crefname{section}{Sec.}{Secs.}
\Crefname{section}{Section}{Sections}
\Crefname{table}{Table}{Tables}
\crefname{table}{Tab.}{Tabs.}

\def\cvprPaperID{9979} % *** Enter the CVPR Paper ID here
\def\confName{CVPR}
\def\confYear{2024}

\renewcommand{\vec}[1]{\boldsymbol{#1}}
\newcommand{\todo}[1]{{\color{red} #1}}
\newcommand{\tp}{\intercal}

\newcommand{\RR}{\mathbb R}
\newcommand{\ee}{\boldsymbol{\varepsilon}}
\newcommand{\xx}{\boldsymbol{x}}
\newcommand{\vt}{\boldsymbol{t}}
\newcommand{\yy}{\boldsymbol{y}}
\newcommand{\XX}{\boldsymbol{X}}

\newcommand{\zz}{\vec{z}}
\newcommand{\CC}{\vec{C}}
\newcommand{\vl}{\vec{\lambda}}
\newcommand{\vu}{\vec{\mu}}

\newcommand{\reals}{\mathbb{R}}
\newcommand{\JJ}{\vec{J}}
\newcommand{\Sig}{{\boldsymbol{\Sigma}}}
\newcommand{\norm}[1]{\left\|#1\right\|}

\newcommand{\corr}[1]{{ #1}}

\newtheorem{theorem}{Theorem}
\numberwithin{theorem}{section}
\newtheorem{proposition}[theorem]{Proposition}
\newtheorem{lemma}[theorem]{Lemma}

\theoremstyle{definition}

\newtheorem{remark}[theorem]{Remark}
\newtheorem{example}[theorem]{Example}

\newcommand{\FR}[1]{{\color{orange}{\textbf{FR:} #1 }}}

\newcommand{\VL}[1]{{\color{magenta}\textbf{VL:} #1}}
\newcommand{\AT}[1]{{\color{teal}\textbf{AT:} #1}}

\begin{document}

\title{Multi-dimensional Sampson Approximations}
\title{Revisiting Sampson Approximations for Geometric Estimation Problems}
\author{Felix Rydell\\
KTH Royal Institute of Technology\\
{\tt\small felixry@kth.se}
\and
Angélica Torres\\
Max Planck Institute for Mathematics in the Sciences\\
{\tt\small angelica.torres@mis.mpg.de}
\and
Viktor Larsson \\
Lund University \\
{\tt\small viktor.larsson@math.lth.se}
}
\maketitle

%%%%%%%%% ABSTRACT
\begin{abstract}
Many problems in computer vision can be formulated as geometric estimation problems, i.e.~given a collection of measurements (e.g.~point correspondences) we wish to fit a model (e.g.~an essential matrix) that agrees with our observations.
This necessitates some measure of how much an observation ``agrees" with a given model.
A natural choice is to consider the smallest perturbation that makes the observation exactly satisfy the constraints.
However, for many problems, this metric is expensive or otherwise intractable to compute. The so-called Sampson error approximates this geometric error through a linearization scheme. For epipolar geometry, the Sampson error is a popular choice and in practice known to yield very tight approximations of the corresponding geometric residual (the reprojection error).
%Many problems in computer vision can be formulated as geometric estimation problems, i.e.~given a collection of measurements (e.g.~point correspondences) we wish to fit a model (e.g.~an essential matrix) that agrees with our observations.
%This necessitates some measure of how much an observation ``agrees" with a given model.
%A natural choice is to consider what is the smallest perturbation that makes the observation exactly satisfy the model.
%However, for many problems this metric is expensive or otherwise intractable to compute. 
%In this case the so-called Sampson approximation can be used which through a linearization scheme approximates the geometric error.
%For epipolar geometry, the Sampson error is a popular choice and in practice known to yield very tight approximations of the corresponding geometric residual (the reprojection error).

In this paper we revisit the Sampson approximation and provide new theoretical insights as to why and when this approximation works, as well as provide explicit bounds on the tightness under some mild assumptions.
Our theoretical results are validated in several experiments on real data and in the context of different geometric estimation tasks. 

\end{abstract}

\iffalse
\section*{TODO-List}
\todo{
\begin{itemize}
    \item Finish teaser figure
    \item Related work
    \item Make supplementary material references vague
\end{itemize}
}
\fi

%%%%%%%%% BODY TEXT
\section{Introduction}
\label{sec:intro}

Estimating a geometric model from a collection of measurements is a common task in computer vision pipelines.
Prerequisite for any such estimation is the ability to check whether a measurement is consistent with a given model.
This can be used both to filter outlier measurements or as a loss for non-linear model refinement.
 For generative models, i.e.~models that can produce the idealized measurements, it is straight-forward to check this consistency, e.g.~computing the difference between the projection and observed 2D keypoint.
However, in many cases the relation between model and data is implicitly encoded through a set of geometric constraints.
One such example is the epipolar constraint in two-view geometry,
\begin{equation}
    C(\xx_1,\xx_2,\vec{E}) = (\vec{x}_2; 1)^\tp \vec{E} (\vec{x}_1; 1) = 0
\end{equation}
relating the essential (or fundamental) matrix $E$ with the point correspondence ($\vec{x}_1,\vec{x}_2) \in \RR^{2}\times\RR^2$.
For a noisy correspondence the constraint will not be satisfied exactly, so it becomes natural to ask: How close is this match to agreeing with the model?
Mathematically, this can be formulated as
\begin{align} \label{eq:optimal_triangulation}
    \min_{\hat{\xx}_1,~\hat{\xx}_2} &~\|\hat{\xx}_1-\xx_1\|^2 + \|\hat{\xx}_2-\xx_2\|^2 \\
    \text{s.t. } &~ C(\hat{\xx}_1,\hat{\xx}_2,\vec{E}) = 0%\nonumber
\end{align}
i.e., what is the smallest pertubation to the points such that they exactly satisfy the constraint.
This minimum distance gives a measure of the consistency between the model (the essential matrix $E$) and the measurements ($\xx_1$ and $\xx_2$.)

\begin{figure}[t]
\centering
\includegraphics[width=0.46\textwidth]{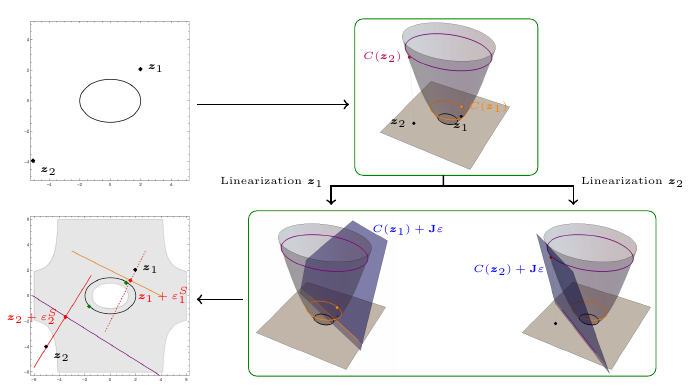}
\caption{\scriptsize The model $C(x,y)= x^2+2y^2-4 =0$ is the ellipse on the top left, and the data points are $\zz_1$ and $\zz_2$. On the top right, the gray surface is the graph $(x,y,C(x,y))$, the orange and purple curves are the level sets $C(x,y)=C(\zz_1)$ and $C(x,y)=C(\zz_2)$ respectively. In the bottom right the blue planes are tangent to \corr{the surface at $(\zz_1,C(\zz_1))$ and $(\zz_2,C(\zz_2))$} respectively. %$C(\zz_1)$ and $C(\zz_2)$ respectively.
The orange and purple lines are the linearized constraints for $\zz_1$ and $\zz_2$ respectively. We keep this color convention for the linearized constraints on the bottom left, and represent their normals in red. The Sampson approximations for $\zz_1$ and $\zz_2$ are the red points \corr{$\zz_1+\ee_1^S$ and $\zz_2+\ee_2^S$}, and %$\ee_1^S$ and $\ee_2^S$ respectively.
the minimizers of the geometric error are depicted in green. %The geometric errors are depicted in green. 
Since $\zz_1$ is in the gray region (obtained from Proposition \ref{prop: eG bound}), its Sampson approximation is better than the approximation for $\zz_2$.} % \FR{Explain what $T_C(z_i)$ is.} \FR{In the top linearization, could we have the same surface patch as the two other ones? Like including the purple circle.}}
\label{fig:teaser}
\end{figure}

For the specific case of the epipolar constraint, the optimal solution to \eqref{eq:optimal_triangulation} can be computed~\cite{hartley1997triangulation,lindstrom2010triangulation} by finding the roots to degree 6 univariate polynomial.  As this procedure is relatively expensive, and difficult to integrate into non-linear refinement methods, it is common in practice to instead use some approximation of the error in \eqref{eq:optimal_triangulation}.
One example of this is the Sampson error~\cite{sampson1982fitting,luong1996fundamental}, defined as
\begin{equation}
    \label{eq:classic_sampson}
     \mathcal{E}^2_{S}= \frac{(({\vec{x}}_2; 1)^\tp \vec{E} ({\vec{x}}_1; 1))^2}{\|\vec{E}_{12}({\vec{x}}_1; 1)\|^2 + \|(\vec{E}^\tp)_{12}({\vec{x}}_2; 1)\|^2}.
\end{equation}
This expression is derived by linearizing the constraint in \eqref{eq:optimal_triangulation}, i.e.~by replacing $C(\hat{\xx}_1,\hat{\xx}_2,E) = 0$ by the constant and linear terms of its Taylor expansion around the data point $(\xx_1,\xx_2)$. It provides a good approximation of \eqref{eq:optimal_triangulation} if the curvature of the constraint at the data point is small enough in relation to the size of $C(\xx_1,\xx_2,E)$, while being cheap to compute and easy to optimize in non-linear refinement.
%which provides a good approximation of \eqref{eq:optimal_triangulation} while being cheap to compute and easy to optimize in non-linear refinement.
%The expression in \eqref{eq:classic_sampson} is derived by applying the Sampson approximation (linearizing the constraint) to \eqref{eq:optimal_triangulation}. 
This approach was originally used by Sampson~\cite{sampson1982fitting} to approximate distances between conics and points, but has since been applied to many other geometric models.

In this paper we revisit the Sampson approximation and provide new theoretical insights into why (and when) the approximation works well.
Under relatively mild assumptions we derive explicit bounds on the tightness of the approximation.
These bounds are experimentally validated on real data and showcased in multiple applications from geometric computer vision (two- and three-view estimation, vanishing point estimation and resectioning).

The paper is organized as follows:  Section~\ref{sec:related_work} discusses the related work. In Section~\ref{sec:classic_sampson} we first present the classical derivation of the Sampson error, along with some geometrical interpretations of the approximation. 
In Section~\ref{sec:bounds} we present our bounds on the approximation. 
Finally, in Section~\ref{sec:experiments} we provide some experimental evaluation. %\FR{We should present the sections in order.} \AT{I just changed it :)}

% Polynomial interpolation: https://www.geogebra.org/m/PkBMaCvy

%\begin{tikzpicture}
%  \draw[->] (-3, 0) -- (4.2, 0); %node[right] {$x$};
 % \draw[->] (0, -3) -- (0, 4.2) node[above] {$y$};
 % \draw[scale=0.5, domain=-1.5:15, smooth, variable=\x, blue] plot ({\x-5}, {5*((0.001451)*\x*\x*\x*\x-(0.039009)*\x*\x*\x+(0.352467)*\x*\x-(1.193414)*\x+(0.878505))});
 % \node at (0,0) [circle,fill,inner sep=1.5pt]{};
  %\draw[scale=0.5, domain=-3:3, smooth, variable=\y, red]  plot ({\y*\y}, {\y});
%\end{tikzpicture}

\subsection{Related Work} \label{sec:related_work}
%\AT{Revisit the idea of moving the related work to the introduction?}
\paragraph{The Geometric Error.} In applied algebraic geometry, fitting noisy data points to a mathematical model defined by polynomials has recently seen a lot of interest \cite{draisma2016euclidean}, and more specifically for 3D reconstruction \cite{maxim2020euclidean,rydell2023theoretical}. One of the main contributions is the development and computation of the so called Euclidean distance degree, which is the number of critical points to the closest point optimization problem given generic data. It expresses the algebraic complexity of fitting data to a model; the higher the Euclidean distance degree is, the more computationally expensive this optimization is. It is used to implement efficient solvers in Homotopy Continuation \cite{HomotopyContinuation} or for solving the associated polynomials systems via specialized symbolic solvers \cite{larsson2017efficient}.

%gives an upper bound on the number of closest point in a model, and it is used to implement efficient solvers in Homotopy Continuation \cite{HomotopyContinuation}. \FR{Solving the associated polynomials systems is typically done either via specialized solvers or homotopy continuation methods.}\FR{ References?}   

\paragraph{The Sampson Error.}
The Sampson approximation was first proposed in~\cite{sampson1982fitting} to approximate the point-conic distance. 
It was also later derived independently by Taubin~\cite{taubin1991estimation}.
Since then it has appeared in numerous papers for different problems.
Luong and Faugeras~\cite{luong1996fundamental} introduced it for approximating the reprojection error in epipolar geometry.
The corresponding geometric error for homographies was first introduced by Sturm~\cite{sturm1997vision} and later revisited by Chum et al.~\cite{chum2005geometric} who also dervied the Sampson approximation in this setting.
Leonardos et al.~\cite{leonardos2015metric} used the Sampson error for point-line-line constraint from the trifocal tensor.
Chojnacki et al.~\cite{chojnacki2000fitting} considered how to integrate known measurement covariances, and recently Terekhov et al.~\cite{terekhov2023tangent} generalized the Sampson error in epipolar geometry to handle arbitrary central camera model.
In~\cite{szpak2012guaranteed} it was used in an optimization method for conic fitting that guarantees convergence to an ellipse. 
This extensive use of the Sampson approximation for geometric problems shows its versatility, which motivates a deeper theoretical study in a more general setting. This is precisely our goal in this work.  

%\AT{Highlight that we give theoretical bounds.} \todo{Add: Why are we listing these references - how do they relate to the contributions in this paper.}

\section{The Sampson Approximation} \label{sec:classic_sampson}
%Traditional Sampson error~\cite{sampson1982fitting,luong1996fundamental} approximates the geometric error, i.e.~the distance to the closest pair of points which satisfy the epipolar constraint. 
%The type of problems we consider in this paper can be formulated as follows
%In general this can be formulated as follows
We consider geometric residuals analogous to \eqref{eq:optimal_triangulation} for general models, i.e.~problems of the form
\begin{align} \label{eq:geometric_err}
  \mathcal{E}_{G}^2(\zz,\vec{\theta}) =\min_{\ee}\quad & \norm{\ee}^2 \\
  \text{s.t. }\quad & C(\zz + \ee, \vec{\theta}) = 0 \label{eq:geometric_err_C}
\end{align}
where \corr{$\zz \in \reals^n$} %$\zz \in \reals^N$
is our measurement, $\vec{\theta}$ our model parameters, and $C(\zz,\vec{\theta})$ our geometric constraint.
In the special case of epipolar geometry, we have that $\zz = [\xx_1,\xx_2]$ consists of the two matching image points, \corr{the model parameters are the entries of the essential matrix $\vec{E}$,} %our model is the essential matrix $\vec{E}$ 
and the constraint is,
\begin{equation}
   C(\zz, \vec{E}) = (\vec{{x}}_2; 1)^\tp \vec{E} (\vec{{x}}_1; 1)=0.
\end{equation}
In the rest of the paper we consider a general polynomial constraint $C(\zz,\vec{\theta})$ and will for notational convenience drop the dependence on the model parameters $\vec{\theta}$ in most places.

Since~\eqref{eq:geometric_err} often does not admit a simple closed form solution, the idea in \cite{sampson1982fitting} is to linearize the constraint, $C(\zz + \ee) = 0$, at the original measurement $\zz$, i.e.
\begin{align} \label{eq:lin_geometric_err}
  %\mathcal{E}_{S}^2(\zz,\vec{E})=\min_{\ee} \quad & \norm{\ee}^2 \\
  \corr{\mathcal{E}_{S}^2(\zz,\vec{\theta})}=\min_{\ee} \quad & \norm{\ee}^2 \\
  \text{s.t. }\quad &  C(\zz) + \vec{J} \ee = 0 \label{eq:lin_geometric_err_C}
\end{align}
where $\vec{J}=\partial C(\vec{z})/\partial \vec{z}$ is the Jacobian of the constraint, evaluated at $\zz$. Introducing a Lagrangian for \eqref{eq:lin_geometric_err},\footnote{For convenience we introduce 1/2 here; it does not affect the optimum.}
\begin{equation}
    %\mathcal{L}(\ee,\lambda) = \frac1{2}\norm{\ee}^2 + \lambda \left( C(\zz) + \vec{J} \ee \right)
    \corr{\mathcal{L}(\ee,\alpha) = \frac1{2}\norm{\ee}^2 + \alpha \left( C(\zz) + \vec{J} \ee \right)}
\end{equation}
we get the first-order constraints as
\begin{align}
     %\ee + \lambda \vec{J}^\tp = 0, \quad 
    \ee + \corr{\alpha \vec{J}^\tp = 0,} \quad 
    C(\zz) + \vec{J}\ee = 0.
\end{align}
Inserting the first equation into the second yields
\begin{equation}
    %C(\zz) - \lambda \|\vec{J}\|^2 = 0 \implies \ee = - \frac{C(\zz)}{\|\vec{J} \|^2}  \vec{J}^\tp.
    \corr{C(\zz) - \alpha \|\vec{J}\|^2 = 0} \implies \ee = - \frac{C(\zz)}{\|\vec{J} \|^2}  \vec{J}^\tp.
\end{equation}
Thus the minimum in \eqref{eq:lin_geometric_err} is given by
\begin{equation}
    %\mathcal{E}_{S}^2 = \left\|\frac{C(\zz)}{\|\vec{J} \|^2}  \vec{J}^\tp \right\|^2 = \frac{C(\zz)^2}{\|\vec{J}\|^2}.
    \corr{\mathcal{E}_{S}^2(\zz,\vec{\theta})} = \left\|\frac{C(\zz)}{\|\vec{J} \|^2}  \vec{J}^\tp \right\|^2 = \frac{C(\zz)^2}{\|\vec{J}\|^2}.
\end{equation}
In the special case of epipolar geometry, replacing $C(\zz)$ with the epipolar constraint, we arrive at the classical formula for the Sampson error \eqref{eq:classic_sampson}.

\subsection{Multiple Constraints and Covariances}\label{ss: mult const}
%In the previous section we saw the Sampson approximation for the case where we have single constraint. 
The Sampson error for a single constraint, discussed previously,  can easily be generalized to the case where we measure the deviation in the Mahalanobis distance from a model defined by multiple constraints. See \cite[p128]{hartley2003multiple} for more details.
Assume that the $N$ constraints are given by
\begin{equation}
    \CC(\zz) = \left(C_1(\zz), C_2(\zz),~\dots,C_N(\zz)\right)^\tp = \vec{0},
\end{equation}
and we want to measure the deviation $\ee$ in Mahalanobis distance for a given \corr{positive definite} covariance $\Sig$. This can be formulated as
\begin{align}%\label{eq:geometric_err}
  %\mathcal{E}_{G}^2(\zz,\vec{E}) =\min_{\ee}\quad & \norm{\ee}^2_\Sig \\
  \corr{\mathcal{E}_{G}^2(\zz,\vec{\theta})} =\min_{\ee}\quad & \norm{\ee}^2_\Sig \\
  \text{s.t. }\quad & \vec{C}(\zz + \ee) = \vec{0}, %\nonumber
\end{align}
where $\norm{\ee}^2_\Sig = \ee^\tp \Sig^{-1} \ee$. The Lagrangian of the linearized constraint then becomes
\begin{equation}
    %\mathcal{L}(\ee,\vl) = \frac1{2} \norm{\ee}^2_\Sig + \vl^\tp \left( \CC(\zz) + \vec{J} \ee  \right)
    \mathcal{L}(\ee,\vec{\alpha}) = \frac1{2} \norm{\ee}^2_\Sig + \vec{\alpha}^\tp \left( \CC(\zz) + \vec{J} \ee  \right),
\end{equation}
where $\vec{J} \in \reals^{N\times n}$ is the Jacobian of $\vec{C}$ evaluated at $\zz \in \reals^n$.
First order conditions again yield
\begin{align}
%     \dfrac{\partial \mathcal{L}}{\partial \ee} &= \Sig^{-1}\ee + \JJ^\tp \vl = \vec{0} \\
%     \dfrac{\partial \mathcal{L}}{\partial \vl} &= \CC(\zz) + \JJ \ee = \vec{0}
     %\Sig^{-1}\ee + \JJ^\tp \vl = \vec{0}, \quad
     \Sig^{-1}\ee + \corr{\JJ^\tp \vec{\alpha}} = \vec{0}, \quad
    \CC(\zz) + \JJ \ee = \vec{0}.
\end{align}
We get $\corr{\ee = -\Sig \JJ^\tp \vec{\alpha},}$ %$\ee = -\Sig \JJ^\tp \vl$,
allowing us to solve for $\corr{\vec{\alpha}}$ as
\begin{equation}
    %\CC(\zz) - \JJ \Sig \JJ^\tp \vl = \vec{0} \implies \vl = (\JJ \Sig \JJ^\tp)^{-1} \CC(\zz).
    \corr{\CC(\zz) - \JJ \Sig \JJ^\tp \vec{\alpha} = \vec{0} \implies \vec{\alpha} = (\JJ \Sig \JJ^\tp)^{-1} \CC(\zz)},
\end{equation}
%This step depends on 
when $\JJ\Sig\JJ^\tp$ is invertible (in \Cref{ss: rank-def} we consider the general case). Let $\ee^S$ denote the argmin of the linearized optimization problem. From this we get
%This step only works if the rows of $\JJ$ are linearly independent, meaning when $\JJ\Sig\JJ^\tp$ is invertible (see \Cref{ss: rank-def} for rank-deficient $\JJ$).
%\AT{The previous matrix is not invertible if there are more constraints than the dimension of the space.} \FR{We assume $\JJ$ is full-rank and then pick out constraints corresponding to linearly independent vectors. This makes the Sampson error not unique and it matters how the choice of generators are made.}
\begin{equation}
    \ee^S = -\Sig \JJ^T (\JJ \Sig \JJ^\tp)^{-1} \CC(\zz),
\end{equation}
%Thus
%\begin{align}
%    \norm{\ee}^2_\Sig &= \| \Sig^{-1/2} \ee \|^2  \\ 
%    &= \| \Sig^{1/2} \JJ^\tp (\JJ\Sig \JJ^\tp)^{-1} \CC(\zz) \|^2 \\
%    &= \| (\JJ\Sig^{1/2})^\dagger \CC(\zz) \|^2.  \label{eq:multi-dim-sampson}
%\end{align}
%where $(\cdot)^\dagger$ denotes the Moore-Penrose psuedo-inverse.
%\VL{Maybe we should only report \eqref{eq:cjjc} here first, and then in the general case arrive at the psuedo-inverse formula?} 
and finally, $\|\ee^S\|_\Sig^2 = \ee^S\Sig^{-1}\ee^S$, which simplifies to
\begin{align} \label{eq:cjjc}
    \CC(\zz)^\tp (\JJ\Sig\JJ^\tp)^{-1}\CC(\zz) = \| \CC(\zz) \|^2_{\JJ\Sigma\JJ^\tp}.
\end{align}

%%%%%%%%%%%%%%%%%%%%%%%%%%%%%%%%%%%%%%%%%%%%%%%%%%%%

\subsection{General Case}\label{ss: rank-def}
Let $(\cdot)^\dagger$ denote the Moore-Penrose psuedo-inverse of a matrix. Assuming that the linearized equation $\CC(\zz)+\JJ\ee=\vec{0}$ has at least one solution, meaning that $\CC(\zz)\in \mathrm{Im}\;\JJ=\mathrm{Im}\;\JJ\Sig^{1/2}$, then $(\JJ\Sig^{1/2})(\JJ\Sig^{1/2})^\dagger\CC(\zz)=\CC(\zz)$.
In this case, all solutions to the linearized equation can be written
\begin{equation}
    \ee(\vu) = - \Sig^{1/2}(\JJ\Sig^{1/2})^\dagger\CC(\zz)+M\vu, \quad \vu \in \reals^N %N\vec{\lambda}, \quad \vl \in \reals^m
\end{equation}
where the columns of $M$ %$N$ 
is a basis for the nullspace of $\JJ$. In order to find $\ee^S$ we consider
\begin{align}
   \|\ee^S\|_\Sig &= \min_{\vu} \|\Sig^{-1/2}\left( - \Sig^{1/2}(\JJ\Sig^{1/2})^\dagger\CC(\zz)+M\vu\right)\|\\
  & = \min_{\vu} \|- (\JJ\Sig^{1/2})^\dagger\CC(\zz)+\Sig^{-1/2}M\vu\|.
\end{align}
Here we note that $M^\tp\Sig^{-1/2} (\JJ\Sig^{1/2})^\dagger=\vec{0}$, which implies that the optimal choice is $\vu=\Vec{0}$. To be precise, we get
\begin{align}
    \ee^S=- \Sig^{1/2}(\JJ\Sig^{1/2})^\dagger\CC(\zz).
\end{align}
and thus
\begin{equation}
    \|\ee^S\|^2_\Sig = \|(\JJ\Sig^{1/2})^\dagger\CC(\zz)\|^2.
\end{equation}

\begin{remark}
For the case of $\Sigma=I$, we note that $\|\ee^S\|=\|\JJ^\dagger\CC(\zz_0)\|$ is the length of the Gauss-Newton step,
\begin{equation}
    \zz_1 = \zz_0 - (\JJ^\tp\JJ)^{-1}\JJ^\tp \CC(\zz_0)
\end{equation}
applied to solving $\CC(\xx) = 0$ starting from the point $\zz_0$. 
Figure~\ref{fig:teaser} visualizes this for one constraint in two variables. 
\end{remark}

\section{Bounding the Approximation Error} \label{sec:bounds}

%In this section we investigate the Sampson approximation~\eqref{eq:lin_geometric_err} in terms of how well it approximates the geometric error \eqref{eq:geometric_err}. We do this by constructing explicit bounds. We start in \Cref{ss: quad - 1 const} by studying one quadratic constraint. In \Cref{ss: high deg - 1 const}, we generalize our methods by considering constraints of higher degrees. Finally, we discuss the case of multiple constraints in \Cref{ss: many const}, but leave the details for the Supplementary Material. 
In this section we investigate the Sampson approximation~\eqref{eq:lin_geometric_err} in terms of how well it approximates the geometric error \eqref{eq:geometric_err}. We do this by constructing explicit bounds. We start by studying one quadratic constraint and then we generalize our methods by considering constraints of higher degrees. Finally, we discuss the case of multiple constraints, but leave the details in the Supplementary Material. 

%\todo{Can we say something about $|\mathcal{E}_{ML} - \mathcal{E}_S |$, or something similar ? Probably in terms of how good the initial point is, i.e.~how close is $\zz$ to satisfy the constraint $C(\zz) = 0$. I guess by making some assumptions on the jacobian, or the hessian we can get something here.}
% \textcolor{blue}{E: From the Taylor series expansions, the error scales proportionally with the step size $h$ raised to the power of $(n + 1)$, where in our specific scenario, with the first-order approximation, $h^2$ is given by $ (\frac{f(x)}{\|\nabla(f)\|})^2.$ Thus, the error cannot exceed this limit.}

%%%%%%%%%%%%%%%%%%%%%%%%%%%%%%%%%%%%%%%%%%%%%%%%%%%%%%%%%%%%%%%%%%%%%%%%%%%%%%%%%%%%%%%%%%%%%%%%%%%%%%%%%%%%%%%%%%%%%%%%%%%%%%%%%%%%%%%%%%%%%%%%%%%%%%%%%%%%%%%%%%%%%%%%%%%%%%%%

\subsection{One Quadratic Constraint}\label{ss: quad - 1 const}
Consider the case where we have one quadratic constraint $C(\zz)$.
The classical Sampson error \eqref{eq:classic_sampson} falls into this category as the epipolar constraint is a quadratic polynomial in terms of the image points. 

We write $\ee^G$ for the argmin of \eqref{eq:geometric_err} and $\ee^S$ for the argmin of \eqref{eq:lin_geometric_err}. %We denote our one quadratic constraint by $C:\RR^n\to \RR$. 
Given a fixed data point $\zz\in \RR^n$, we write $\Vec{H}$ for the Hessian of $C$ at $\zz$. Since $C$ is a quadratic polynomial, 
\begin{align} \label{eq:C_quadratic}
    C(\zz+\ee)=C(\zz)+\JJ\ee +\frac{1}{2}\ee^\tp \Vec{H}\ee.
\end{align}
We use that any vector $\ee$ satisfies the inequality
\begin{align}
    |\ee^\tp \Vec{H} \ee|\le \rho \|\ee\|^2
\end{align}
where $\rho$ is the spectral radius of the Hessian of $C(\zz)$, that is, the maximum of the absolute values of its eigenvalues.

First we present an upper bound on the Sampson approximation $\|\ee^S\|$ in terms of the geometric residual $\|\ee^G\|$. %, and the basic idea is to take $\ee^G$ and orthogonally project it onto the affine linear space defined by the linearized constraint $C(\zz)+\JJ\ee=0$.

\begin{proposition}\label{prop: eS lower} When the optimization problem \eqref{eq:geometric_err} only has one quadratic constraint and $\JJ\neq 0$, then
    \begin{align}
        \|\ee^S\|\le \|\ee^G\|+ \frac{\rho}{2\|\JJ\|}\|\ee^G\|^2.
    \end{align}
\end{proposition}

Under the assumption that $\rho/2\|\JJ\|$ is reasonably small (i.e.~function is approximately linear), we can interpret this proposition as saying that if $\ee^S$ is big, then $\ee^G$ is also big.

The condition $\JJ\neq 0$ is reasonable, as the approximation is undefined if the linearized constraint set is empty. %This relation implies that if $\|\ee^S\|$ is big, then $\|\ee^G\|$ is also big, so having a high error in the linearization problem guarantees that the geometric error is also high. \VL{Well I guess it depends on $\rho/2\|J\|$ ? Maybe this is very big but $\|\ee^G\|$ is small?} \FR{Yes, I meant in relative terms.}

\begin{proof} Since $\ee^G$ satisfies the constraint \eqref{eq:C_quadratic}, we have that
\begin{align}%\begin{aligned}
    0=&~C(\zz)+\JJ\ee^G+\frac{1}{2}(\ee^G)^\tp \Vec{H}\ee^G\\
    =&~  C(\zz)+\JJ\Big(\ee^G+\frac{\JJ^\tp}{\|\JJ\|^2}\frac{1}{2}(\ee^G)^\tp \Vec{H}\ee^G\Big).
%\end{aligned}
\end{align}
In other words, $ \ee^G+\frac{\JJ^\tp}{\|\JJ\|^2}\frac{1}{2}(\ee^G)^\tp \Vec{H}\ee^G$ satisfies the linearized constraint. 
Since $\ee^S$ is by definition the smallest vector that satisfies the linearized constraint, 
we get that
\begin{align}
    \|\ee^S\|\le \|\ee^G\|+\frac{1}{2\|\JJ\|}|(\ee^G)^\tp \Vec{H}\ee^G|.
\end{align}
Finally, using the inequality of the spectral radius $\rho$ mentioned above, we get the upper bound.
\end{proof}

Next we give an upper bound of $\|\ee^G\|$ in terms of $\|\ee^S\|$.
\begin{proposition}\label{prop: eG bound}
    If the minimization problem \eqref{eq:geometric_err} only has one quadratic constraint, 
    \begin{align}\label{eq: C(z) above}
       \JJ\neq 0 \textnormal{ and } \|\JJ\|^4\ge 2|C(\zz)||\JJ \Vec{H}\JJ^\tp|,
    \end{align}
    then 
    \begin{align}\label{eq:eG_bound}
        \|\ee^G\|\le 2\|\ee^S\|.
    \end{align}
\end{proposition}
The assumption should be intuitively understood as the model being close enough to linear in the direction of $\JJ$ locally around the input $\zz$, relative to the size of $C(\zz)$. 
Note that from the Proposition, we also have
\begin{align}
    \|\ee^S-\ee^G\|\le \|\ee^S\|+\|\ee^G\|\le3\|\ee^S\|. 
\end{align}

%\VL{Can also be that you are close to the constraint set, i.e. C(\zz) small. (relative to the slope). Could potentially use the stronger assumption $\|\JJ\|^2 \ge 2\rho|C(\zz)|$ that might be better for presentation. Edit: ah or this is exactly the orange region below. I see} \FR{This constraints is harder to compute so I think it is actually less natural.}

\begin{proof}
    %If $C(\zz)=0$, then $\ee^S=\ee^G=0$. Assume that $C(\zz)\neq0$. 
    Define 
    %\begin{align}
     $   \ee(\lambda):=\frac{\lambda}{\|\JJ\|} \JJ^\tp$
    %\end{align}
    for $\lambda\in\mathbb{R}$ and consider
    %We wish to find a solution $\lambda^*$ to 
    \begin{equation}\label{eq: C z ee}
         C(\zz + \ee(\lambda))= C(\zz) + \|\JJ\| \lambda + \frac{\JJ \Vec{H} \JJ^\tp}{2\|\JJ\|^2} \lambda^2,
    \end{equation}
     which is a quadratic polynomial in $\lambda$.
    %In this way we parametrize the line spanned by $\JJ^\tp$ and we are looking for $\lambda^*$ such that $C(\zz+\ee(\lambda^*))=0$. 
 %   Then, $\|\ee^G\|\le \|\ee(\lambda^*)\|=|\lambda^*|$ by construction. 
%
%\VL{Remove this maybe (comes again later in the text):\textit{
 %   We find $\lambda^*$ by solving the equation
 %   \begin{equation}
  %      0 = C(\zz + \ee(\lambda))= C(\zz) + \|\JJ\| \lambda + \frac{\JJ H \JJ^\tp}{2\|\JJ\|^2} \lambda^2,
   % \end{equation}
   % which is quadratic in $\lambda$ as long as $\JJ H\JJ^\tp\neq 0$.
    %}}
    %We find $\lambda^*$ by solving the equation
    %\begin{equation}
     %   0 = C(\zz + \ee(\lambda))= C(\zz) + \|\JJ\| \lambda + \frac{\JJ H \JJ^\tp}{2\|\JJ\|^2} \lambda^2,
    %\end{equation}
    %which is quadratic in $\lambda$ as long as $\JJ H\JJ^\tp\neq 0$.
    Evaluating this polynomial at $\lambda= -2C(\zz)/\|\JJ\|$, we get 
    \begin{align}
        -C(\zz)+2\frac{\JJ \Vec{H}\JJ^\tp}{\|\JJ\|^4}C(\zz)^2.
    \end{align}
    The absolute value of the second term can by assumption be bounded from above by $|C(\zz)|$.
    %This means that if $C(\zz) > 0$ then
    %$$C(\zz+\ee( -2C(\zz)/\|J\|))\le 0 $$
    This means that either $C(\zz)\ge 0$ and $C(\zz+\ee( -2C(\zz)/\|\JJ\|))\le 0$ or the other way around.
    By continuity of polynomials, 
    there must exist a solution $\lambda^*$ to \eqref{eq: C z ee} in the interval $\Big[0,-2|C(\zz)|/\|\JJ\|\Big]$.
    %there is a $\lambda^*$ solution to $C(\zz+\ee(\lambda^*))=0$ in the interval $\Big[0,-2|C(\zz)|/\|J\|\Big]$.
    Since $\ee^G$ is the smallest vector satisfying this, we have
    \begin{equation}
    \|\ee^G\|\le \|\ee(\lambda^*)\|=|\lambda^*| \le 2|C(\zz)|/\|\JJ\| = 2\|\ee^S\|.
    \end{equation}
\end{proof}

In order to get a sharper bound in the case that $\JJ \vec{H}\JJ^\tp\neq 0$, we may instead solve the quadratic equation \eqref{eq: C z ee} directly: 
    \begin{equation} \label{eq:lambda_sol}
        \lambda^*= \frac{-\|\JJ\|^3 \pm \|\JJ\|\sqrt{\|\JJ\|^4-2C(\zz)\JJ \Vec{H} \JJ^\tp}}{\JJ \vec{H} \JJ^\tp}.
    \end{equation} 
    Let $\lambda^*$ be the solution with smallest absolute value, then $\|\ee^G\|\le |\lambda^*|$ and, following from the proof of \Cref{prop: eG bound}, $ |\lambda^*|\le 2\|\ee^S\|$.
    Note that if $\JJ \Vec{H}\JJ^\tp = 0$, then $\lambda^* = -C(\zz)/\|\JJ\|$ and we directly get $\|\ee^G\| \le \|\ee^S\|$.

To summarize we have the following inequalities
\begin{equation}
    \frac1{2}\|\ee^G\| \le
    \frac1{\tau} \|\ee^G\| \le
    \|\ee^S\| \le \|\ee^G \| + \frac{\rho}{2\|\JJ\|}\|\ee^G\|^2,
\end{equation}
where $\tau = |\lambda^*/C(\zz)|\,\|\JJ\|$ and $\lambda^*$ is given by \eqref{eq:lambda_sol}.

%\[
%\frac1{2} \mathcal{E}_G \le \mathcal{E}_S \le \mathcal{E}_G + \frac1{2}\frac{\|\vec{H}\|}{\|\JJ\|}\mathcal{E}_G^2
%\]

The hypothesis in \Cref{prop: eG bound} defines a region where the geometric error is bounded linearly by the Sampson error. We highlight that this region contains the region defined by
\begin{equation}\label{eq:eG_relaxed_bound}
    \|\ee^S\| = \frac{|C(\zz)|}{\|\JJ\|}\leq \frac{\|\JJ\|}{2\rho}.
\end{equation}
\begin{example}\label{ex:ellipse}

    %The results in the previous propositions hold for any constraint defined by a quadric. 
    Consider as a constraint the conic $C(\xx)=x_1^2+2x_2^2-4$ in $\reals^2$. For fixed $\zz=(z_1,z_2)$, the minimization problem is %then 
    \begin{align} 
        \min_{\ee}\quad & \norm{\ee}^2 \\
        \text{s.t. }\quad & (z_1+\varepsilon_1)^2+2(z_2+\varepsilon_2)^2-4=0 %\nonumber
    \end{align}
    and the linearized constraint is
    \begin{equation}
        z_1^2+2z_2^2-4 + \begin{bmatrix}
             2z_1& 4z_2
         \end{bmatrix}\begin{bmatrix}
             \varepsilon_1 \\ \varepsilon_2
         \end{bmatrix}=0.
    \end{equation}
    % \begin{align} 
    %     \min_{\ee}\quad & \norm{\ee}^2 \\
    %     \text{s.t. }\quad & z_1^2+2z_2^2-4 + \begin{bmatrix}
    %         2z_1& 4z_2
    %     \end{bmatrix}\begin{bmatrix}
    %         \varepsilon_1 \\ \varepsilon_2
    %     \end{bmatrix}=0. \nonumber
    % \end{align}
    From \Cref{prop: eG bound}, we have that $\|\ee^G\|\leq 2\|\ee^S\|$ if $\zz$ is in the colored region of~\Cref{fig:1const_ellipse} (both the orange and purple region). Finally, the relaxed condition from~\eqref{eq:eG_relaxed_bound} gives a smaller region, potentially easier to work with, where the previous bounds also hold (purple region in \Cref{fig:1const_ellipse}) %\VL{Update colors}
    \begin{figure}
        \centering
        %\begin{subfigure}[b]{0.45\textwidth}
        %\includegraphics[scale=0.58]{figs/One_constraint/SE_regions_ellipse.pdf}
        %\includegraphics[width=0.4\textwidth]{figs/One_constraint/SE_heatmap_ellipse.pdf}
        \includegraphics[height=0.45\columnwidth]{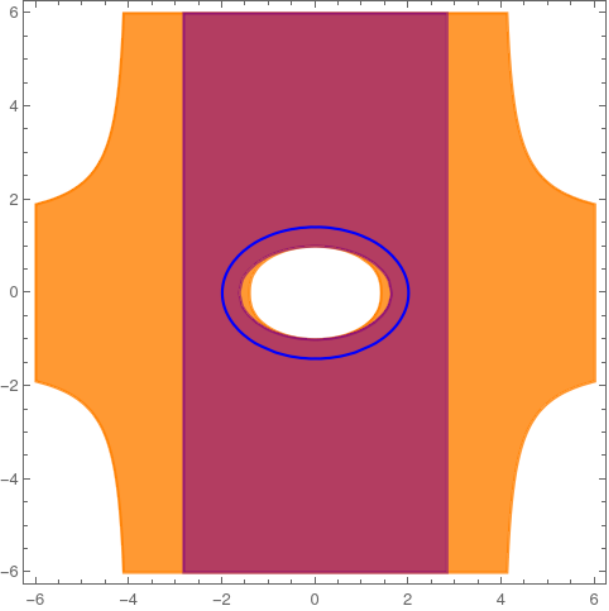}
        \includegraphics[height=0.45\columnwidth]{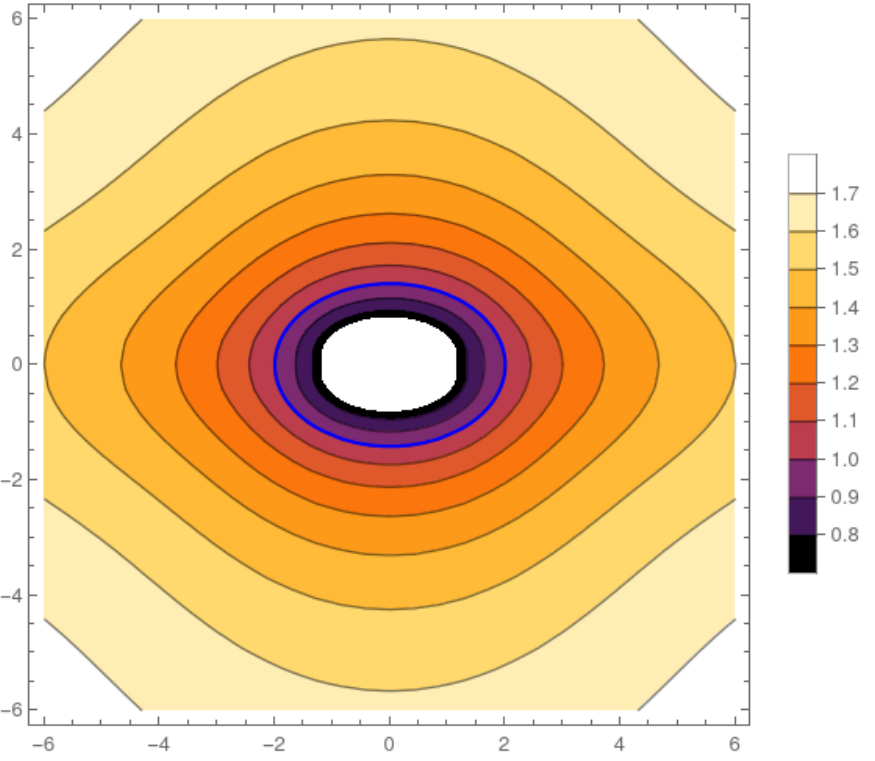}
        
        %\end{subfigure}
        % \begin{subfigure}[b]{0.45\textwidth}
        %     \includegraphics[scale=0.3]{figs/One_constraint/SE_true_regions_ellipse.pdf}
        % \includegraphics[scale=0.3]{figs/One_constraint/SE_relaxed_regions_ellipse.pdf}
        % \caption{On the left, the true regions arising from Propositions \ref{prop: eG bound} and \ref{prop:eG_continuity_bound}. On the right, the regions obtained from the relaxations. }
        % \end{subfigure}
        \caption{\footnotesize Regions obtained for \Cref{ex:ellipse}. On the left, the constraint  $x_1^2+2x_2^2-4=0$ is depicted in blue, the purple region is obtained from \Cref{eq:eG_relaxed_bound} and it is contained in the orange region coming from \Cref{eq:eG_bound}. On the right, the level sets for the ratio $\|\ee^G\|/\|\ee^S\|$, and the constraint depicted in blue. Observe that, although the ratio changes, it is bounded by 2 for every point in the colored regions.  }%On the right, the orange region from Propositions \ref{prop: eG bound} and \ref{prop:eG_continuity_bound}, and the red region corresponds to the relaxation given by \Cref{eq:eG_relaxed_bound}. \VL{maybe show a heatmap of the approximation errors ?} \FR{Perhaps the left plot is superfluous}}
        \label{fig:1const_ellipse}
    \end{figure}
    
\end{example}

%%%%%%%%%%%%%%%%%%%%%%%%%%%%%%%%%%%%%%%%%%%%%%%%%%%%%%%%%%%%%%%%%%%%%%%%%%%%%%%%%%%%%%%%%%%%%%%%%%%%%%%%%%%%%%%%%%%%%%%%%%%%%%%%%%%%%%%%%%%%%%%%%%%%%%%%%%%%%%%%%%%%%%%%%%%%%%%%%%%%%%%%%%%%%%%%%%%%%%%%%%%%%%%%%

\subsection{One Polynomial Constraint} \label{ss: high deg - 1 const}

To understand the Sampson error for a model defined by a single polynomial constraint of any degree $C:\RR^n\to \RR$, we extend the method presented previously and consider Taylor approximations of degree $d$:
\begin{align}
    C(\zz+\ee)=C(\zz)+\sum_{i=1}^d \frac{1}{i!} \ee \times \mathcal T_i, 
\end{align}
where $\mathcal T_i$ is a symmetric $n\times \cdots \times n$ %\AT{Is this $N$ related to the number of constraints in the multiple constraint case?} 
tensor of order $i$, and 
\begin{align}
    \ee \times \mathcal T_i=\sum_{j_1,\ldots,j_i\in [n]} (\mathcal T_i)_{j_1,\ldots,j_i}\varepsilon_{j_1}\cdots \varepsilon_{j_i}.
\end{align}
For example, $\mathcal T_1=\JJ$ is the Jacobian and $\mathcal T_2=\Vec{H}$ the Hessian.

\begin{proposition}\label{prop: eG d bound} When the optimization problem \eqref{eq:geometric_err} only has one polynomial constraint of degree $d$, 
\begin{align}\label{eq: C(z) above T}
       \JJ\neq 0 \textnormal{ and } |C(\zz)|\ge\Big|\sum_{i=2}^d(-1)^i \frac{2^iC(\zz)^i}{i!\|\JJ\|^i} \JJ\times \mathcal T_i\Big|,
    \end{align}
then 
%\begin{align}
$
\|\ee^G\|\le 2\|\ee^S\|.
$
%\end{align}
\end{proposition}

\begin{proof} The proof is exactly the same as for \Cref{prop: eG bound}; one checks that $C(\zz)\le 0$ and $C(\zz+\ee(\lambda))\ge 0$ (or the other way around) for $\lambda =-2C(\zz)/\|\JJ\|$.
\end{proof}

%%%%%%%%%%%%%%%%%%%%%%%%%%%%%%%%%%%%%%%%%%%%%%%%%%%%%%%%%%%%%%%%%%%%%%%%%%%%%%%%%%%%%%%%%%%%%%%%%%%%%%%%%%%%%%%%%%%%%%%%%%%%%%%%%%%%%%%%%%%%%%%%%%%%%%%%%%%%

\subsection{Multiple Polynomial Constraints}\label{ss: many const} For mathematical models defined by multiple constraints, the Sampson error and its relation to the geometric error are more involved. Here, we give an overview of our approach to computing Sampson errors in practice and studying its relation to the geometric error for polynomial constraints, but leave mathematical statements and proofs in the Supplementary Material.
%Here, we discuss general ideas and leave mathematical statements and proofs in the Supplementary Material. 

%Given $N$ constraints $\CC(\zz)=(C_1(\zz),\ldots,C_N(\zz))$ with $\zz=(z_1,\ldots ,z_n)$, assume that these constraints define a variety $X$ (a zero set of a system of polynomials) of dimension $d$. Intuitively, the more constraints we have, the smaller is linearized constraint set. Therefore, if we $n$ is much greater than $N-d$, the Sampson approximation is likely to be poor. To remedy this, note that it is a fact of algebraic geometry that locally around a generic point of $X$, the variety is described by precisely $N-d$ polynomials, whose Jacobian is full-rank. We therefore propose to perform Sampson approximation by first choosing a subset of constraints $C_i$ of size $N-d$ for which the Jacobian is full-rank. It can also be beneficial to choose these constraints of lowest possible degree. 
%Given $N$ polynomial constraints $\CC(\zz)=(C_1(\zz),\ldots,C_N(\zz))$, 

Assume that  $\CC(\zz)=(C_1(\zz),\ldots,C_N(\zz))$ are $N$ polynomial constraints. The model $\CC(\zz)=\Vec{0}$, for $\zz\in \RR^n$, is an algebraic variety $X$ (a zero set of a system of polynomials) of some dimension $m\le n$, which depends on the constraints. 
For generic constraints of fixed degrees $m=n-N$, i.e. each constraint lowers the dimension of the model by one, however, for specific systems of polynomials this equality does not necessarily hold. In both cases, the intuitive behaviour is that the more constraints we have, the smaller $m$ is and our data is harder to fit to the model. The same phenomenon occurs with the linearized constraints for the Sampson error: The more constraints we have, the smaller is the affine linear space defined by $\CC(\zz)+\JJ\ee=0$. Therefore, if the number of constraints $N$ is much greater than the codimension $n-m$, the Sampson approximation is likely to be poor. To remedy this, we propose to use the fact that locally around a generic point of $X$, the model is described by precisely $n-m$ polynomials whose Jacobian has full rank. This is a result coming from algebraic geometry. Our proposal is to perform Sampson approximation by choosing a subset of $n-m$ constraints whose Jacobian is full-rank and linearizing those constraints. In \Cref{s: 3-view} we try this approach for the Three-View Sampson error and the results suggest that it is also beneficial to choose these constraints with smallest degree possible. 

%We claim that it should also be beneficial to choose these constraints to be of smallest possible degree, which our experiments in \Cref{s: 3-view} motivate.% that it can also be beneficial to choose these constraints with the lowest possible degree.
%It can also be beneficial to choose these constraints of lowest possible degree. 
%We show this empirically in the experiment in \Cref{s: 3-view}.

%The next question we are concerned with is how to generalize the results of this section to multiple bounds. We can essentially generalize \Cref{prop: eS lower} directly to multiple constraints of any degrees. For upper bounds of $\|\ee^G\|$, we restrict to quadratic constraints and data points $\zz$ such that the Jacobian $\JJ$ at $\zz$ has linearly independent rows. In this case,
In the Supplementary Material, we generalize \Cref{prop: eS lower} and find an upper bound for $\|\ee^S\|$ in terms of $\|\ee^G\|$ for multiple constraints of any degrees. We also generalize \Cref{prop: eG bound} for quadratic constraints and data points $\zz$ such that the Jacobian $\JJ$ at $\zz$ has linearly independent rows. In this case, under appropriate conditions, we find a $\vl^*$ such that $C_i(\zz+\|\JJ\|\JJ^\dagger \vl^*)=0$ for $i=1,\ldots,n$.  %$C_i(\zz+\ee(\vl^*))=0$ for $\ee(\vl)=\|\JJ\|\JJ^\dagger \vl$ and $i=1,\ldots,n$. 
We get 
\begin{align}
    \|\ee^G\|\le \|\|\JJ\|\JJ^\dagger \vl^*\|\le \|\JJ\|\|\JJ^\dagger\|\|\vl^*\|,
\end{align}
and by construction we get an upper bound for this $\|\vl^*\|$ expressed in terms of $\CC$, its Jacobian and its Hessian.

\section{Experimental Evaluation} \label{sec:experiments}
In the following sections we evaluate the Sampson approximation for different geometric estimation problems.
First, in Section~\ref{s: 2-view} we evaluate the classical Sampson error for two-view geometry.
Next, in Section~\ref{s: 3-view} we consider the analogous error in the three-view setting. Section~\ref{s:vp} considers line segment to vanishing point errors. Finally, Section~\ref{s:2d3d} show an application from absolute pose estimation with uncertainties applied in both 2D and 3D.

%%%%%%%

\subsection{Application: Two-view Relative Pose} \label{s: 2-view}
We first consider the classical setting where the Sampson approximation is applied, two-view relative pose estimation.
For the experiments we use image pairs from the IMC Phototourism 2021~\cite{jin2021imc} (SIFT), MegaDepth-1500~\cite{sun2021loftr} (SP+LG~\cite{lindenberger2023lightglue}) and ScanNet-1500~\cite{sarlin2020superglue} (SP+SG~\cite{sarlin2020superglue}).
For each image pair we estimate an initial essential matrix using DLT~\cite{hartley2003multiple} applied to the ground truth inliers (from~\cite{lindstrom2010triangulation}).

\noindent\textbf{Results and Discussion.}
Figure~\ref{fig:relpose_err_dist} shows the distribution of the difference between the Sampson approximation and the true error,
and in Table~\ref{tbl:relpose_err} shows the Area-Under-Curve\footnote{Area under the CDF up to 1 px error as a ratio of the complete square.} up to 1 pixel, on the \textit{British Museum} scene from IMC-PT.
For comparison we also include the symmetric epipolar error (distances to the epipolar lines computed in both images) which is another popular choice in practice.
The Sampson error provides a very accurate approximation of the true reprojection error.
As discussed in Section~\ref{sec:bounds} the quality of the approximation depends on the how close the initial point is to the constraint set (small $C(\zz)$) and the curvature (small Hessian). 
In Figure~\ref{fig:relpose_err_vs_bnd} we plot the approximation error $|\mathcal{E}_S-\mathcal{E}_G|$ against $\rho|C(\zz)|/\|J\|^2$, where $\rho$ is the spectral norm of the Hessian. Consistent with the theory, the figure shows a clear trend where correspondences with smaller $\rho|C(\zz)|/\|J\|^2$ have smaller errors.

Finally we also evaluate the error in the context of pose refinement on all three datasets.
Table~\ref{tbl:new_experiment} shows the resulting pose errors (max of rotation and translation error) after non-linear refinement of the initial essential matrix using different error functions. \textit{Algebraic} is the squared epipolar constraints. \textit{Cosine} is the squared cosine of the angles between the normals of the epipolar planes and the point correspondences.
The Sampson error provides the most accurate camera poses after refinement.

\begin{figure}[ht]
\centering
\scalebox{.7}{\input{figs/relative_pose}}
\caption{\footnotesize\textbf{Approximation gap for two-view relative pose.} Comparison with optimal triangulation~\cite{lindstrom2010triangulation}. The unit is in pixels. Here $\mathcal E$ refers to either the Sampson or the symmetric epipolar error.}
\label{fig:relpose_err_dist}
~\\~
\begin{tabular}{l c c c} \toprule
          & \multicolumn{3}{c}{$|\mathcal{E}-\mathcal{E}_G|$, ~~AUC @ $\tau$ px ($\uparrow$)} \\ \cmidrule{2-4}
    Error      & $\tau = 0.1$ & $\tau = 0.5$ & $\tau = 1$ \\ \midrule
Sampson   & \bf 0.991 &\bf 0.998 &\bf 0.999 \\
Symmetric Epipolar & 0.620 & 0.839 & 0.902 \\ \bottomrule
\end{tabular}
\captionof{table}{\footnotesize The table shows the Area-Under-Curve of the approximation error $|\mathcal{E}-\mathcal{E}_G|$ up to different thresholds.}
\label{tbl:relpose_err}
\end{figure}

\begin{figure}
\centering
\scalebox{.7}{\input{figs/relative_pose_err_vs_bnd}}
\caption{\footnotesize\textbf{Approximation error against $\rho|C(\zz)|/\|J\|^2$ for two-view relative pose.} Figure shows a heatmap built from the inlier correspondences of $\sim$5k image pairs from the \textit{British Museum} scene. Points that are either close to satisfying the epipolar constraint ($C(\zz)\approx 0$) or have low curvature ($\rho$) have smaller errors.}
\label{fig:relpose_err_vs_bnd}
\end{figure}
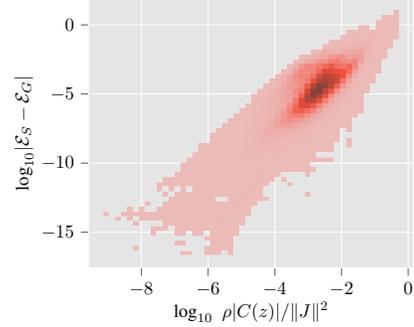

\iffalse
\begin{table}[ht]
\centering
\begin{tabular}{l c c c} \toprule
& \multicolumn{3}{c}{Pose error AUC ($\uparrow$)} \\ \cmidrule{2-4}
& @5$^\circ$ &  @10$^\circ$ & @20$^\circ$ \\ \midrule
%Initial estimate (DLT) & 18.43 & 29.15 & 41.41\\ 
%~\rotatebox[origin=c]{180}{$\Lsh$} Algebraic & 37.54 & 45.86 & 53.26 \\
%~\rotatebox[origin=c]{180}{$\Lsh$} Cosine & 45.49 & 54.08 & 60.83 \\
%~\rotatebox[origin=c]{180}{$\Lsh$} Sym. Epipolar & 46.46 & 55.24 & 62.05 \\
%~\rotatebox[origin=c]{180}{$\Lsh$} Sampson & \bf 46.79 & \bf 55.67 & \bf 62.45 \\ 
Initial estimate (DLT) & 0.184 & 0.292 & 0.414\\ 
~\rotatebox[origin=c]{180}{$\Lsh$} Algebraic & 0.375 & 0.459 & 0.533 \\
~\rotatebox[origin=c]{180}{$\Lsh$} Cosine & 0.455 & 0.541 & 0.608 \\
~\rotatebox[origin=c]{180}{$\Lsh$} Sym. Epipolar & 0.465 & 0.552 & 0.621 \\
~\rotatebox[origin=c]{180}{$\Lsh$} Sampson & \bf 0.468 & \bf 0.557 & \bf 0.625 \\ 
\bottomrule
\end{tabular}
\caption{\footnotesize\textbf{Comparison of losses for two-view relative pose refinement on IMC British Museum.} The initial pose is found by linear estimate (DLT) applied to the inlier correspondences (w.r.t.~ground truth pose), followed by non-linear refinement.}
\label{tbl:imc_bc}
\end{table}
\fi

\begin{table}[ht]
\begin{tabular}{l c c c} \toprule
%& \multicolumn{3}{c}{Pose error AUC ($\uparrow$)} \\ \cmidrule{2-4}
%& @5$^\circ$ &  @10$^\circ$ & @20$^\circ$ \\ \midrule
AUC$@10^\circ$ pose error & IMC-PT & MD1.5k & SN1.5k \\ \midrule
%Initial estimate (DLT) & 18.43 & 29.15 & 41.41\\ 
%~\rotatebox[origin=c]{180}{$\Lsh$} Algebraic & 37.54 & 45.86 & 53.26 \\
%~\rotatebox[origin=c]{180}{$\Lsh$} Cosine & 45.49 & 54.08 & 60.83 \\
%~\rotatebox[origin=c]{180}{$\Lsh$} Sym. Epipolar & 46.46 & 55.24 & 62.05 \\
%~\rotatebox[origin=c]{180}{$\Lsh$} Sampson & \bf 46.79 & \bf 55.67 & \bf 62.45 \\ 
Initial estimate (DLT) & 0.361 & 0.321 & 0.274\\ 
~\rotatebox[origin=c]{180}{$\Lsh$} Algebraic &  0.580 & 0.515 & 0.485\\
~\rotatebox[origin=c]{180}{$\Lsh$} Cosine &  0.654 & 0.689 & 0.600\\
~\rotatebox[origin=c]{180}{$\Lsh$} Sym. Epipolar & 0.673 & 0.728 & 0.654\\
~\rotatebox[origin=c]{180}{$\Lsh$} Sampson & \bf  0.678 & \bf 0.732 & \bf 0.657\\ 
\bottomrule
\end{tabular}
%\caption{More results for the experiment in Section~4.1 (main paper). The table shows the pose AUC@10$^\circ$ for the 9 scenes in the IMC-PT 2021 test set with SIFT matches, the ScanNet1500 benchmark with SuperPoint+SuperGlue matches, and the MegaDepth1500 benchmark with SuperPoint+LightGlue matches.}
\vspace{-0.3cm}
\caption{\footnotesize\textbf{Comparison of losses for two-view relative pose refinement on IMC-PT, MegaDepth-1500 and ScanNet-1500.} The initial pose is found by linear estimate (DLT) applied to the inlier correspondences (w.r.t.~ground truth pose), followed by non-linear refinement.}
\label{tbl:new_experiment}
\end{table}

%%%%%%%%%%%%%%%%%%%%%%%%%%

\subsection{Application: Three-view Sampson Error}\label{s: 3-view}
\newcommand{\Ts}{\mathtt{T}}
In this section we evaluate different error formulations for 3-view point matches. The naive baseline is to simply average the two-view Sampson errors \eqref{eq:classic_sampson},  $\quad \mathcal{E}_{pair} =~~$
\begin{equation} \label{eq:3view_pair}
         \mathcal{E}_S(\xx,\xx',E_{12}) + \mathcal{E}_S(\xx,\xx'',E_{13}) +  \mathcal{E}_S(\xx',\xx'',E_{23})
\end{equation}
where $(\xx,\xx',\xx'')$ is the correspondence and $E_{ij}$ are the essential matrices. For a given a trifocal tensor $\mathcal{T} \in \reals^{3\times 3\times 3}$ with slices $\Ts_1, \Ts_2, \Ts_3 \in \reals^{3\times 3}$, a consistent three-view point correspondence $\xx, \xx', \xx''$ satisfies
\begin{equation}\label{eq:3view}
   \CC_9(\xx,\xx',\xx'') = [\hat{\xx}']_\times \left( \sum_k \hat{\xx}_k \Ts_k \right) [\hat{\xx}'']_\times = \vec{0}
\end{equation}
where $\hat{\xx} = [\xx; 1]$ is the homogenization of the 2D point $\xx \in \reals^2$. While we have nine equations, only four are linearly independent.
These can be obtained by multiplying with two $3\times 2$ matrices,
\begin{equation}\label{eq:3view_4}
   \CC_4(\xx,\xx',\xx'') = S_1^\tp [\hat{\xx}']_\times \left( \sum_k \hat{\xx}_k \Ts_k \right) [\hat{\xx}'']_\times S_2
\end{equation}
where $S_1, S_2 \in \reals^{3\times 2}$ contains a basis for the complement of the left and right nullspace of $\CC_9$ respectively. 
Another alternative is to only consider the three pairwise  constraints, 
\begin{equation} \small \label{eq:3view_3}
    \CC_3(\xx,\xx',\xx'') = \left[ \xx'^\tp E_{12} \xx,~\xx''^\tp E_{13}\xx,~\xx''^\tp E_{23}\xx' \right]
\end{equation}
Note that applying the Sampson approximation to $\CC_3$ yields a different error compared to $\mathcal{E}_{pair}$ \eqref{eq:3view_pair}, as the constraints are considered jointly.

In this section we evaluate the following errors
\begin{itemize} \itemsep0pt
    \item $\mathcal{E}_{pair}$ - averaging the pairwise Sampson errors  \eqref{eq:3view_pair}
    \item $\mathcal{E}_{9}$ - applying Sampson approximation to $\CC_9$ \eqref{eq:3view} 
    \item $\mathcal{E}_{4}$ - applying Sampson approximation to $\CC_4$ \eqref{eq:3view_4}
    \item $\mathcal{E}_{3}$ - applying Sampson approximation to $\CC_3$ \eqref{eq:3view_3}
\end{itemize}
We also include two combinations of the above. First taking 3 out of the 4 constraints from $\CC_4$, denoted $\CC_{4:3}$, and one where we combine $\CC_3$ and $\CC_4$ by taking two quadratic and one cubic constraint, denoted $\CC_{4:1,3:2}$.
We also consider a set of psuedo-Sampson approximations which take the form $\|\CC\|^2 / \|\JJ\|^2$, and thus avoid computing matrix inverses as in \Cref{ss: rank-def}. This can be seen as a naive extension of the 1-dimensional Sampson approximation \eqref{eq:lin_geometric_err} to the multi-dimensional case. We denote these as $\hat{\mathcal{E}_9}, \hat{\mathcal{E}_4}, \hat{\mathcal{E}_3}$.

\noindent\textbf{Experiment setup.}
To compare the approximations we generate synthetic camera triplets (70$^\circ$ field-of-view, 1000x1000 pixel images) observing a 3D point.
To the three projections we add normally distributed noise with standard deviation $\sigma\in\{1,5,10\}$ px.
We obtain the reference ground-truth reprojection error by directly optimizing over the 3D point. 
For each of the evaluated approximated error functions we compute the difference to the ground-truth.

\noindent\textbf{Results and Discussion.}
Table~\ref{tbl:trifocal} shows the errors for 100k synthetic  instances.
We compute the Area-Under-Curve up to 1 pixel deviation from the reference error.
The Sampson approximation of the epipolar constraints $\CC_3$ \eqref{eq:3view_3} yields the best approximation.
In particular, we can see that the naive approximation $\mathcal{E}_{pair}$ that averages the pairwise Sampson errors, is significantly worse.
Interestingly, we can also see that $\CC_4$ performs much worse compared to both $\CC_3$ and the mixed variants $\CC_{4:3}$ and $\CC_{4:1,3:2}$. 
This is consistent with the discussion in \cref{ss: many const}.

\begin{table}[ht]
\centering
\resizebox{\columnwidth}{!}{
\begin{tabular}{l c c c} \toprule
    &  \multicolumn{3}{c}{$|\mathcal{E}-\mathcal{E}_{gt}|$,\quad AUC @ 1px $(\uparrow)$} \\ \cmidrule{2-4}
    & $\sigma=1$px & $\sigma=5$px & $\sigma=10$px \\ \midrule
$\mathcal{E}_3 = \|J_3^\dagger\CC_3\|$      & \bf 0.998 & \bf 0.961 & \bf 0.882 \\
$\mathcal{E}_{4:1,3:2}= \small \|J_{4:1,3:2}^\dagger\CC_{4:1,3:2}\|$   &    0.995 & 0.937 & 0.830 \\

$\mathcal{E}_{4:3}= \|J_{4:3}^\dagger\CC_{4:3}\|$       & 0.994 & 0.912 & 0.768 \\
$\mathcal{E}_4= \|J_4^\dagger\CC_4\|$      & 0.765 & 0.481 & 0.353 \\
$\mathcal{E}_{pair} = ... $ \eqref{eq:3view_pair} & 0.764 & 0.312 & 0.172 \\
$\hat{\mathcal{E}_3} = \|\CC_3\|/\|J_3\|$ & 0.361 & 0.016 & 0.003 \\
$\hat{\mathcal{E}_4} = \|\CC_4\|/\|J_4\|$ & 0.356 & 0.009 & 0.001 \\
$\hat{\mathcal{E}_9} = \|\CC_9\|/\|J_9\|$ & 0.356 & 0.009 & 0.001 \\
$\mathcal{E}_9 = \|J_9^\dagger\CC_9\| $      & 0.348 & 0.075 & 0.034 \\
%alg                  & 0.175 & 0.035 & 0.018 \\
\bottomrule
\end{tabular}
}
\caption{\footnotesize\textbf{Comparison of three-view reprojection error approximations.} Each error metric is compared to the ground truth error (found by non-linear refinement of the 3D point). To each keypoint we add noise with standard deviation $\sigma$ pixels.}
\label{tbl:trifocal}
\end{table}

%%%%%%%%%%%%%%%%%%%%%%%%%%%%%%%%%%%%%%%%%%%%%%%%%%%%%%%%%%%%%%%%%%%%%%%%%%%%%%%%%%%%%%%%%%%%%%%%%%%%%%%%%%%%%%%%%%%%%%%%%%%%%%%%%%%%%%%%%%%%%%%%%%%%%%%%%%%%%%%%%%%%%%%%%

%%%%%%%%%%%%%%%%%%%%%%%%%%%%%%%%%%%%%%%%%%%%%%%%%%%%%%%%%%%%%%%%%%%%%%%%%%%%%%%%%%%%%%%%%%%%%%%%%%%%%%%%%%%%%%%%%%%%%%%%%%%%%%%%%%%%%%%%%%%%%%%%%%%%%%%%%%%%%%%%%

\begin{table*}[ht]
    \centering
    \resizebox{\textwidth}{!}{
    \begin{tabular}{l ccc l ccc l ccc r} \toprule
         %& \multicolumn{2}{c}{Chess} & \multicolumn{2}{c}{Fire} & \multicolumn{2}{c}{Heads} & \multicolumn{2}{c}{Office} & \multicolumn{2}{c}{Pumpkin} & \multicolumn{2}{c}{Redkitchen} & \multicolumn{2}{c}{Stairs} & \multicolumn{3}{c}{Average} \\
         & \multicolumn{3}{c}{$\tau = 5$ px} && \multicolumn{3}{c}{$\tau = 10$ px} && \multicolumn{3}{c}{$\tau = 20$ px} &  \\ \cmidrule{2-4} \cmidrule{6-8} \cmidrule{10-12}
         & AUC $(\uparrow)$ & $\varepsilon_{R}$ $(\downarrow)$ & $\varepsilon_{t}$ $(\downarrow)$ && AUC $(\uparrow)$ & $\varepsilon_{R}$ $(\downarrow)$ & $\varepsilon_{t}$ $(\downarrow)$  && AUC $(\uparrow)$ & $\varepsilon_{R}$ $(\downarrow)$ & $\varepsilon_{t}$ $(\downarrow)$ & RT  \\ \midrule        
     Reprojection error \eqref{eq:normal_bundle} &  0.358 &  1.03 & 3.20 && 0.350 &  1.06 & 3.24 && 0.311 &  1.21 & 3.63  & 0.7 ms  \\
     Reprojection error + Cov. \eqref{eq:full_bundle} & \bf 0.367 & \bf  1.01 &\bf  3.10 && \bf 0.379 &   0.99 & \bf 2.99 && \bf 0.378 & \bf  0.99 &\bf  3.00 & 10.1 ms \\
     ~\rotatebox[origin=c]{180}{$\Lsh$} Sampson approximation &\bf  0.367 &\bf   1.01 &\bf  3.10 &&\bf  0.379 &  \bf 0.98 & \bf 2.99 && 0.375 &  1.00 & 3.02  & 2.4 ms \\\bottomrule
    \end{tabular}}
    \caption{\footnotesize\textbf{Pose refinement on 7Scenes.} Table shows the Area-Under-Curve (AUC) @ 5cm position errors and the median errors in rotation $\varepsilon_R$ (deg.) and translation $\varepsilon_t$ (cm). Results are reported for different inlier thresholds $\tau$. For large thresholds, more uncertain points are included and the improvement from the covariance weighting is more significant. However, larger errors also make the linearization point (the original correspondence) worse, degrading the Sampson approximation. For the two lower thresholds, the Sampson approximation of 2D-3D covariance weighted reprojection error gives almost identical results as performing the full (expensive) optimization.
    }
    \label{tbl:7scenes}
\end{table*}

\subsection{Application: Vanishing Point Estimation} \label{s:vp}

We now show another example with a 1-dimensional quadratic constraint. 
Consider a line-segment $(\xx_1,\xx_2) \in \reals^2 \times \reals^2$ and a vanishing point $\vec{v} \in \mathbb{S}^2$.
Assuming we want to refine $\vec{v}$, it is reasonable to consider what is the smallest pertubation of the line endpoints such that the line passes through the vanishing point, i.e. satisfy the constraint
\begin{equation}
    C(\xx_1,\xx_2) = \vec{v}^\tp ( \begin{pmatrix} \xx_1 \\ 1\end{pmatrix} \times \begin{pmatrix}  \xx_2 \\ 1\end{pmatrix}  )
\end{equation}
Differentiating with respect to the image points we get,
\begin{equation} \small
    J = 
    \begin{bmatrix}
        \left(\begin{pmatrix} \xx_2 \\ 1\end{pmatrix} \times \vec{v}\right) S, & \left(\vec{v} \times
        \begin{pmatrix} \xx_1 \\ 1\end{pmatrix} \right) S
    \end{bmatrix}
\end{equation}
where $S = \begin{bmatrix} I_2 & \vec{0} \end{bmatrix}^\tp$, and we can directly setup a Sampson approximation of the line-segment to vanishing point distance as $|C(\xx_1,\xx_2)|/\|\JJ\|$.
For this problem the ground-truth error $\mathcal{E}_G$ can be computed in closed form using SVD.
%For this we can directly derive a Sampson approximation, with

% [cross(x2,v)'*[eye(2); 0 0], cross(v,x1)'*[eye(2); 0 0]

%\begin{equation} \small
%H_C = 
%\begin{pmatrix}
%& & & v_3 \\
%& & -v_3 &  \\
%& -v_3 & &  \\
%v_3 & & & \\
%\end{pmatrix}
%\end{equation}

In Section~\ref{ss: quad - 1 const} we derived bounds that relate the true geometric error $\ee^G$ and the Sampson approximation $\ee^S$,
\begin{equation}
    B_l := \frac1{\tau} \le \frac{\|\ee^S\|}{\|\ee^G\|} \le 1 + \frac{\rho}{2\|\JJ\|} \|\ee^G\| =: B_u
\end{equation}
Next we evaluate how tight these bounds are on real data. 

\noindent\textbf{Experiment Setup.}
For the experiment we consider circa 350k pairs of line segments and vanishing points collected from the YUB+~\cite{yorkurban} and NYU VP~\cite{silberman_2012,kluger2020consac}.
The line segments are detected using DeepLSD~\cite{pautrat2023deeplsd} and using Progressive-X~\cite{Barath_2019_ICCV} we estimate a set of vanishing points.

\noindent\textbf{Results and Discussion.}
Figure~\ref{fig:vp_bounds} shows $B_u, B_l$ for each of line-vanishing point pairs and in Figure~\ref{fig:vp_bounds_diff} we show the distribution of the difference between the bounds.
As can be seen in the figures the approximation works extremely well for this setting.
We also experimented with refining the vanishing points using the Sampson error but found that the results are very similar to minimizing the mid-point error (as was done in~\cite{pautrat2023deeplsd}).
The full results and details can be found in the Supplementary Material.

\begin{figure}
\centering
\scalebox{.7}{\input{figs/vp_bounds}}
\caption{\footnotesize\textbf{Evaluation of the bounds for VP-line error.}  The lower bound $B_l$ and upper bound $B_u$ for the $\approx$350k VP-line pairs in the combined dataset. For illustration the pairs are sorted w.r.t.~the tightness of the bound. }
\label{fig:vp_bounds}
\scalebox{.7}{\input{figs/vp_bounds_diff}}
\caption{\footnotesize \textbf{Evaluation of the bounds for VP-line error.} Distribution of the $\text{log}_{10}$ differences between the upper and lower bounds.}
\label{fig:vp_bounds_diff}
\end{figure}

%%%%%%%%%%%%%%%%%%%%%%%%%%%%%%%%%%%%%%%%%%%%%%%%%%%%%%%%%%%%%%%%%%%%%%%%%%%%%%%%%%%%%%%%%%%%%%%%%%%%%%%%%%%%%%%%%%%%%%%%%%%%%%%%%%%%%%%%%%%%%%%%%%%%%%%%%%%%%%%%%

\subsection{Application: 2D/3D Reprojection Error} \label{s:2d3d}

Minimizing the square reprojection error, i.e.~deviation between the observed 2D point and the projection of the 3D point, assumes a Gaussian noise model on the 2D observations.
However, in many scenarios, we also have noise in the 3D points. In this section we consider the case where we have a known covariance for both the 2D and 3D-points.

Given a point correspondence $(\xx,\XX) \in \reals^2 \times \reals^3$, together with covariances $\Sig_{2D} \in \reals^{2\times 2}$ and $\Sig_{3D} \in \reals^{3\times 3}$, the maximum likelihood estimate is then given by
\begin{align}
    \min_{\ee_{2},\ee_{3}}\quad  &  \norm{\ee_{2}}_{\Sig_{2D}}^2 +  \norm{\ee_{3}}_{\Sig_{3D}}^2 \\
    \text{s.t.}\quad  & \CC(\xx + \ee_{2}, \XX + \ee_{3}) = 0 
\end{align}
%
%\begin{align}
%    \min_{\ee_{2}}\quad  &  \norm{\ee_{2}}^2  \\
%    \text{s.t.}\quad  & \CC(\xx + \ee_{2}, \XX) = 0 
%\end{align}
%\begin{align}
%    \min_{\ee_{2},\ee_{3}}\quad  &  \norm{\ee_{2}}^2 +  \norm{\ee_{3}}^2 \\
%    \text{s.t.}\quad  & \CC(\xx + \ee_{2}, \XX + \ee_{3}) = 0 
%\end{align}
%
where $\CC(\xx,\XX)$ encodes the reprojection equations, i.e.~
\begin{equation} \label{eq:reproj_constr}
    \CC(\xx,\XX) = [I_{2\times 2},~ -\xx]\left( R \XX + \vt \right) = 0.
\end{equation}
\noindent\textbf{Experiment Setup.}
For the experiment we consider the visual localization benchmark setup on 7Scenes~\cite{shotton2013scene} dataset.
Using HLoc~\cite{sarlin2019coarse} we establish 2D-3D matches and estimate an initial camera pose for each query image.
For the experiment we then refine this camera pose, including the uncertainty in both the 2D and 3D points. The 2D covariances are assumed to be unit gaussians and to obtain the 3D covariances we propagate the 2D covariances from the mapping images used to triangulate the 3D point.

\noindent\textbf{Results and Discussion.}
Table~\ref{tbl:7scenes} shows the average pose error across all scenes (per-scene results are available in the Supplementary Material). 
We compare only minimizing the reprojection error (only 2D noise)
\begin{equation} \label{eq:normal_bundle}
    \min_{R,\vec{t}} \sum_k \|\xx_k - \pi(R\XX_k + \vec{t})\|^2_{\Sig_2}
\end{equation}
with optimizing over the 3D points as well (2D/3D noise),
\begin{equation} \label{eq:full_bundle}
    \min_{R,\vec{t},\{\hat{\XX}_k\}} \sum_k \|\xx_k - \pi(R\hat{\XX}_k + \vec{t})\|^2_{\Sig_2} + \|\hat{\XX}_k - \XX_k\|^2_{\Sig_3}
\end{equation}
and applying the Sampson approximation to \eqref{eq:reproj_constr}.
As shown in Table~\ref{tbl:7scenes}, including the uncertainty of the 3D point can greatly improve the pose accuracy. 
Further, the Sampson approximation works well in this setting and it is only when we include matches with very large errors (20 pixels) that performance degrades.

Note that the optimization problem in \eqref{eq:full_bundle} requires parameterizing each individual 3D point, potentially leading to hundreds or thousands of extra parameters compared to \eqref{eq:normal_bundle} that only optimize over the 6-DoF in the camera pose.
Since the Sampson approximation eliminates the extra unknowns, it also allows us to only optimize over the camera pose while modelling the 3D uncertainty.
Table~\ref{tbl:7scenes} also shows the average runtime in milliseconds for the query images.
Minimizing the Sampson approximation is significantly faster compared to \eqref{eq:full_bundle}.

\iffalse
\begin{figure}[ht]
    \input{figs/7scenes_runtime}
    \caption{\footnotesize \textbf{Runtime comparison on 7Scenes}. The average runtime in milliseconds for the pose refinement experiment on 7Scenes.}
    \label{fig:7scenes}
\end{figure}
\fi

\section{Conclusions}
The Sampson approximation, originally applied to compute conic-point distances, has shown itself to surprisingly versatile in the context of robust model fitting.
While it has been known that it works extremely well in practice, we provide the first theoretical bounds on the approximation error.
In multiple experiments on real data in different application contexts we have validated our theory and highlighted the usefulness of the approximation.
\\~\\
{\footnotesize
\noindent\textbf{{Acknowledgments:}} Viktor Larsson was supported by the strategic research project ELLIIT and the Swedish Research Council (grant no.~2023-05424). Felix Rydell was supported by the Knut and Alice Wallenberg Foundation within their WASP (Wallenberg AI, Autonomous
Systems and Software Program) AI/Math initiative. Angélica Torres was supported by DFG grant 464109215 within the priority programme SPP 2298 ``Theoretical Foundations of Deep Learning”. 
}
%%%%%%%%% REFERENCES
{\small
\bibliographystyle{ieee_fullname}
\bibliography{main}
}

\appendix
\appendixpage

\input{arXiv_supp}

\end{document}

%% file: figs/relative_pose.tex
% This file was created with tikzplotlib v0.10.1.
\begin{tikzpicture}

\definecolor{chocolate2267451}{RGB}{226,74,51}
\definecolor{dimgray85}{RGB}{85,85,85}
\definecolor{gainsboro229}{RGB}{229,229,229}
\definecolor{gray119}{RGB}{119,119,119}
\definecolor{mediumpurple152142213}{RGB}{152,142,213}
\definecolor{steelblue52138189}{RGB}{52,138,189}

\begin{axis}[
width=\columnwidth,
height=0.6\columnwidth,
legend cell align={left},
legend style={at={(0.03,0.8)},anchor=west},
axis background/.style={fill=gainsboro229},
axis line style={white},
tick align=outside,
tick pos=left,
xlabel = {$\text{log}_{10} |\mathcal{E}-\mathcal{E}_G|$},
x grid style={white},
xmajorgrids,
xmin=-17.6342150828323, xmax=3.12241543300041,
xtick style={color=dimgray85},
y grid style={white},
ymajorgrids,
ymin=0, ymax=0.406842867603392,
ytick style={color=dimgray85}
]
\draw[draw=none,fill=chocolate2267451,fill opacity=0.5,very thin] (axis cs:-16.6907318775671,0) rectangle (axis cs:-15.8043045158377,1.94325395762289e-05);
\draw[draw=none,fill=chocolate2267451,fill opacity=0.5,very thin] (axis cs:-15.8043045158377,0) rectangle (axis cs:-14.9178771541083,8.06450392413498e-05);
\draw[draw=none,fill=chocolate2267451,fill opacity=0.5,very thin] (axis cs:-14.9178771541083,0) rectangle (axis cs:-14.0314497923788,0.000611153369672398);
\draw[draw=none,fill=chocolate2267451,fill opacity=0.5,very thin] (axis cs:-14.0314497923788,0) rectangle (axis cs:-13.1450224306494,0.00191216189430092);
\draw[draw=none,fill=chocolate2267451,fill opacity=0.5,very thin] (axis cs:-13.1450224306494,0) rectangle (axis cs:-12.25859506892,0.00103866924034943);
\draw[draw=none,fill=chocolate2267451,fill opacity=0.5,very thin] (axis cs:-12.25859506892,0) rectangle (axis cs:-11.3721677071905,0.00121842023142955);
\draw[draw=none,fill=chocolate2267451,fill opacity=0.5,very thin] (axis cs:-11.3721677071905,0) rectangle (axis cs:-10.4857403454611,0.00171880812551744);
\draw[draw=none,fill=chocolate2267451,fill opacity=0.5,very thin] (axis cs:-10.4857403454611,0) rectangle (axis cs:-9.59931298373165,0.00286629958749376);
\draw[draw=none,fill=chocolate2267451,fill opacity=0.5,very thin] (axis cs:-9.59931298373165,0) rectangle (axis cs:-8.71288562200221,0.00652447516271885);
\draw[draw=none,fill=chocolate2267451,fill opacity=0.5,very thin] (axis cs:-8.71288562200221,0) rectangle (axis cs:-7.82645826027277,0.0172667830404582);
\draw[draw=none,fill=chocolate2267451,fill opacity=0.5,very thin] (axis cs:-7.82645826027277,0) rectangle (axis cs:-6.94003089854334,0.0442789846783951);
\draw[draw=none,fill=chocolate2267451,fill opacity=0.5,very thin] (axis cs:-6.94003089854334,0) rectangle (axis cs:-6.0536035368139,0.104317758953112);
\draw[draw=none,fill=chocolate2267451,fill opacity=0.5,very thin] (axis cs:-6.0536035368139,0) rectangle (axis cs:-5.16717617508447,0.196097643371641);
\draw[draw=none,fill=chocolate2267451,fill opacity=0.5,very thin] (axis cs:-5.16717617508447,0) rectangle (axis cs:-4.28074881335503,0.269019219758418);
\draw[draw=none,fill=chocolate2267451,fill opacity=0.5,very thin] (axis cs:-4.28074881335503,0) rectangle (axis cs:-3.39432145162559,0.260378541035848);
\draw[draw=none,fill=chocolate2267451,fill opacity=0.5,very thin] (axis cs:-3.39432145162559,0) rectangle (axis cs:-2.50789408989616,0.16760954035289);
\draw[draw=none,fill=chocolate2267451,fill opacity=0.5,very thin] (axis cs:-2.50789408989616,0) rectangle (axis cs:-1.62146672816672,0.0466332368480553);
\draw[draw=none,fill=chocolate2267451,fill opacity=0.5,very thin] (axis cs:-1.62146672816672,0) rectangle (axis cs:-0.735039366437286,0.00577632238903403);
\draw[draw=none,fill=chocolate2267451,fill opacity=0.5,very thin] (axis cs:-0.735039366437286,0) rectangle (axis cs:0.151387995292151,0.000702486305680673);
\draw[draw=none,fill=chocolate2267451,fill opacity=0.5,very thin] (axis cs:0.151387995292151,0) rectangle (axis cs:1.03781535702158,5.34394838346295e-05);
\draw[draw=none,fill=mediumpurple152142213,fill opacity=0.5,very thin] (axis cs:-16.2891113851567,0) rectangle (axis cs:-15.3657092045121,2.70489132108616e-05);
\draw[draw=none,fill=mediumpurple152142213,fill opacity=0.5,very thin] (axis cs:-15.3657092045121,0) rectangle (axis cs:-14.4423070238675,0.000179082459878808);
\draw[draw=none,fill=mediumpurple152142213,fill opacity=0.5,very thin] (axis cs:-14.4423070238675,0) rectangle (axis cs:-13.5189048432229,0.00110247639362891);
\draw[draw=none,fill=mediumpurple152142213,fill opacity=0.5,very thin] (axis cs:-13.5189048432229,0) rectangle (axis cs:-12.5955026625783,0.000958837337267784);
\draw[draw=none,fill=mediumpurple152142213,fill opacity=0.5,very thin] (axis cs:-12.5955026625783,0) rectangle (axis cs:-11.6721004819337,0.000958837337267784);
\draw[draw=none,fill=mediumpurple152142213,fill opacity=0.5,very thin] (axis cs:-11.6721004819337,0) rectangle (axis cs:-10.7486983012891,0.00110713999935492);
\draw[draw=none,fill=mediumpurple152142213,fill opacity=0.5,very thin] (axis cs:-10.7486983012891,0) rectangle (axis cs:-9.82529612064453,0.000908470395426869);
\draw[draw=none,fill=mediumpurple152142213,fill opacity=0.5,very thin] (axis cs:-9.82529612064453,0) rectangle (axis cs:-8.90189393999992,0.000558699965976073);
\draw[draw=none,fill=mediumpurple152142213,fill opacity=0.5,very thin] (axis cs:-8.90189393999992,0) rectangle (axis cs:-7.97849175935532,0.000318057910513924);
\draw[draw=none,fill=mediumpurple152142213,fill opacity=0.5,very thin] (axis cs:-7.97849175935532,0) rectangle (axis cs:-7.05508957871072,0.000559632687121275);
\draw[draw=none,fill=mediumpurple152142213,fill opacity=0.5,very thin] (axis cs:-7.05508957871072,0) rectangle (axis cs:-6.13168739806612,0.00213313325907726);
\draw[draw=none,fill=mediumpurple152142213,fill opacity=0.5,very thin] (axis cs:-6.13168739806612,0) rectangle (axis cs:-5.20828521742152,0.00770614210165997);
\draw[draw=none,fill=mediumpurple152142213,fill opacity=0.5,very thin] (axis cs:-5.20828521742152,0) rectangle (axis cs:-4.28488303677692,0.0250491590755482);
\draw[draw=none,fill=mediumpurple152142213,fill opacity=0.5,very thin] (axis cs:-4.28488303677692,0) rectangle (axis cs:-3.36148085613232,0.0718512406995004);
\draw[draw=none,fill=mediumpurple152142213,fill opacity=0.5,very thin] (axis cs:-3.36148085613232,0) rectangle (axis cs:-2.43807867548772,0.179753086382208);
\draw[draw=none,fill=mediumpurple152142213,fill opacity=0.5,very thin] (axis cs:-2.43807867548772,0) rectangle (axis cs:-1.51467649484312,0.341489731123692);
\draw[draw=none,fill=mediumpurple152142213,fill opacity=0.5,very thin] (axis cs:-1.51467649484312,0) rectangle (axis cs:-0.591274314198518,0.332415287102021);
\draw[draw=none,fill=mediumpurple152142213,fill opacity=0.5,very thin] (axis cs:-0.591274314198518,0) rectangle (axis cs:0.332127866446083,0.107078252911494);
\draw[draw=none,fill=mediumpurple152142213,fill opacity=0.5,very thin] (axis cs:0.332127866446083,0) rectangle (axis cs:1.25553004709068,0.00873306808252749);
\draw[draw=none,fill=mediumpurple152142213,fill opacity=0.5,very thin] (axis cs:1.25553004709068,0) rectangle (axis cs:2.17893222773528,6.43577590189465e-05);
\addplot [line width=1.5pt, chocolate2267451]
table {%
-16.6907318775671 5.73998291869808e-06
-16.5116556428743 1.36711396242116e-05
-16.3325794081815 2.68828183028884e-05
-16.1535031734887 1.98137354379467e-05
-15.9744269387959 2.5397486211325e-05
-15.7953507041031 2.91805891325128e-05
-15.6162744694102 4.9329967323708e-05
-15.4371982347174 6.33971621857896e-05
-15.2581220000246 7.68032214202583e-05
-15.0790457653318 0.000122993851674291
-14.899969530639 0.000210803348938437
-14.7208932959462 0.000327697688263356
-14.5418170612534 0.000445844662216817
-14.3627408265605 0.000692970823662834
-14.1836645918677 0.00100167002301918
-14.0045883571749 0.00140259343368203
-13.8255121224821 0.00196836419819982
-13.6464358877893 0.00226011181286829
-13.4673596530965 0.00205999972321258
-13.2882834184036 0.00175807326247486
-13.1092071837108 0.00117042168357743
-12.930130949018 0.000920738262568271
-12.7510547143252 0.000962406007386239
-12.5719784796324 0.00105790787416039
-12.3929022449396 0.00114368596652361
-12.2138260102468 0.00109374013669057
-12.0347497755539 0.00113060474625375
-11.8556735408611 0.00117151638413613
-11.6765973061683 0.00128523804932341
-11.4975210714755 0.00140737924612873
-11.3184448367827 0.00144692670024105
-11.1393686020899 0.00148622873581228
-10.960292367397 0.00165368463555398
-10.7812161327042 0.00185544751803305
-10.6021398980114 0.00200028445902099
-10.4230636633186 0.00218162056750132
-10.2439874286258 0.00231500957980959
-10.064911193933 0.00270402808794107
-9.88583495924015 0.00318508508959479
-9.70675872454733 0.00371852501386236
-9.52768248985452 0.00425514388874386
-9.34860625516171 0.00524209278484382
-9.16953002046889 0.00609317420862759
-8.99045378577607 0.00770810846561755
-8.81137755108326 0.00922479010769884
-8.63230131639044 0.0114060840053187
-8.45322508169763 0.0138474785440434
-8.27414884700481 0.0168720594434057
-8.095072612312 0.0199449028399094
-7.91599637761918 0.0241644444698268
-7.73692014292637 0.0295316693336518
-7.55784390823355 0.0357451747675832
-7.37876767354074 0.0435879740924076
-7.19969143884792 0.0523979910532489
-7.0206152041551 0.0622028022359922
-6.84153896946229 0.0744474726433811
-6.66246273476948 0.0878889671185349
-6.48338650007666 0.103396882697169
-6.30431026538384 0.121747573390899
-6.12523403069103 0.140779548987238
-5.94615779599821 0.16020261984416
-5.7670815613054 0.180082016191048
-5.58800532661258 0.201090106674089
-5.40892909191977 0.21762530094094
-5.22985285722695 0.233078862783122
-5.05077662253414 0.250312279353098
-4.87170038784132 0.264675821507516
-4.69262415314851 0.274223241543681
-4.51354791845569 0.281429386783682
-4.33447168376287 0.278184151085137
-4.15539544907006 0.271664666381982
-3.97631921437725 0.266363062304521
-3.79724297968443 0.263938519107922
-3.61816674499161 0.254667608646442
-3.4390905102988 0.231479024142293
-3.26001427560598 0.214288866373802
-3.08093804091317 0.194712712672718
-2.90186180622035 0.16101619431326
-2.72278557152754 0.12949334557496
-2.54370933683472 0.100255968011419
-2.36463310214191 0.0759374670942683
-2.18555686744909 0.0568929056808712
-2.00648063275628 0.0383095254369145
-1.82740439806346 0.0232433624771023
-1.64832816337064 0.015860093367979
-1.46925192867783 0.00974958858540605
-1.29017569398501 0.00611443363548955
-1.1110994592922 0.0043557129602641
-0.932023224599384 0.00289719677048196
-0.752946989906567 0.00210612523149154
-0.573870755213754 0.00111095835220729
-0.394794520520936 0.000534460811476105
-0.215718285828121 0.000641966052407412
-0.0366420511353063 0.000332601403679995
0.142434183557508 0.000151453564317123
0.321510418250323 9.42230577202371e-05
0.500586652943138 5.25800582057106e-05
0.679662887635956 3.11036600157709e-05
0.858739122328771 1.10385493042936e-05
1.03781535702158 1.10110520143355e-05
};
\addlegendentry{Sampson error};
\addplot [line width=1.5pt, mediumpurple152142213]
table {%
-16.2891113851567 8.36213890049704e-06
-16.102565490077 1.32149470694799e-05
-15.9160195949973 1.94538951913241e-05
-15.7294736999176 3.41869763773973e-05
-15.5429278048379 3.46154850423176e-05
-15.3563819097581 4.92439246646283e-05
-15.1698360146784 7.56662229897642e-05
-14.9832901195987 0.000120430163389617
-14.796744224519 0.00020936767654952
-14.6101983294393 0.000302827987193491
-14.4236524343596 0.000424517329920973
-14.2371065392798 0.00066397138215872
-14.0505606442001 0.000952752832872221
-13.8640147491204 0.00135603056201477
-13.6774688540407 0.00162201020686511
-13.490922958961 0.00155587327642498
-13.3043770638813 0.00128815414120113
-13.1178311688015 0.000843330991999
-12.9312852737218 0.000682814725421856
-12.7447393786421 0.000732796703943431
-12.5581934835624 0.000894610064738212
-12.3716475884827 0.000816928151480103
-12.1851016934029 0.000901971902932963
-11.9985557983232 0.00106432993196674
-11.8120099032435 0.00104592276089776
-11.6254640081638 0.00112799174197862
-11.4389181130841 0.00120722329419681
-11.2523722180044 0.00116091565616496
-11.0658263229246 0.00108074393002978
-10.8792804278449 0.00096312032791596
-10.6927345327652 0.000971405977630206
-10.5061886376855 0.00108422733626681
-10.3196427426058 0.000866089063765798
-10.1330968475261 0.000868272354130572
-9.94655095244634 0.000780176340309876
-9.76000505736662 0.000730824491088498
-9.57345916228691 0.00067193885749909
-9.38691326720719 0.000570752498912476
-9.20036737212747 0.000455436892488193
-9.01382147704776 0.000428487194824865
-8.82727558196804 0.00033074678402171
-8.64072968688832 0.000274529612255423
-8.4541837918086 0.000317259192266797
-8.26763789672889 0.000345413630170208
-8.08109200164917 0.000303729644158655
-7.89454610656945 0.000325972787583679
-7.70800021148973 0.000436688286592468
-7.52145431641002 0.000497557013733104
-7.3349084213303 0.000668241724804545
-7.14836252625058 0.000875096156172075
-6.96181663117086 0.00106164761128442
-6.77527073609115 0.00155650939564186
-6.58872484101143 0.00211343289435521
-6.40217894593171 0.00267777005201223
-6.215633050852 0.00344332016995635
-6.02908715577228 0.00431043812151579
-5.84254126069256 0.00564397743489075
-5.65599536561284 0.00750404116605706
-5.46944947053313 0.00959321858633425
-5.28290357545341 0.0123726014175184
-5.09635768037369 0.0154509818899619
-4.90981178529397 0.0196132050125956
-4.72326589021426 0.0250837143196207
-4.53671999513454 0.0311023404628698
-4.35017410005482 0.0379074874734828
-4.1636282049751 0.047293959775841
-3.97708230989539 0.0583924511420605
-3.79053641481567 0.0717417191786884
-3.60399051973595 0.0886475707361244
-3.41744462465624 0.107080690723555
-3.23089872957652 0.127675046238595
-3.0443528344968 0.152322228770436
-2.85780693941708 0.182479409085522
-2.67126104433737 0.2179064174163
-2.48471514925765 0.254225715147923
-2.29816925417793 0.289376118339188
-2.11162335909821 0.322761955292652
-1.9250774640185 0.356968143016441
-1.73853156893878 0.379853530734732
-1.55198567385906 0.387469397717516
-1.36543977877935 0.380015600832094
-1.17889388369963 0.357298667753429
-0.99234798861991 0.335513950362905
-0.805802093540194 0.29548713485755
-0.619256198460477 0.233291295339243
-0.432710303380759 0.178839017792731
-0.246164408301041 0.123260596805933
-0.0596185132213236 0.0805902780210328
0.126927381858394 0.0546199123992611
0.313473276938112 0.0319999503683286
0.500019172017829 0.0181547852160946
0.686565067097543 0.00878369516421022
0.873110962177261 0.00456884863178485
1.05965685725698 0.00127177907366213
1.2462027523367 0.000324956117867613
1.43274864741641 0.000136617020393647
1.61929454249613 6.18322467227376e-05
1.80584043757585 3.04139224493627e-05
1.99238633265557 2.66334300014026e-05
2.17893222773529 1.10377660288996e-05
};
\addlegendentry{Sym. epipolar error};
\end{axis}

\end{tikzpicture}

%% file: figs/relative_pose_err_vs_bnd.tex
% This file was created with tikzplotlib v0.10.1.
\begin{tikzpicture}

\definecolor{dimgray85}{RGB}{85,85,85}
\definecolor{gainsboro229}{RGB}{229,229,229}
\begin{scope}
\begin{axis}[
scale=0.9,
%width=0.9\columnwidth,
%height=0.9\columnwidth,
axis background/.style={fill=gainsboro229},
axis line style={white},
tick align=outside,
tick pos=left,
x grid style={white},
xmajorgrids,
xlabel={$\text{log}_{10}~~{\rho|C(z)|}/{\|J\|^2}$},
ylabel={$\text{log}_{10} |\mathcal{E}_S - \mathcal{E}_G|$},
xmin=-9.57973134926437, xmax=0.157237353875922,
xtick style={color=dimgray85},
y grid style={white},
ymajorgrids,
ymin=-17.5771592392966, ymax=1.92424271875102,
ytick style={color=dimgray85}
]
\addplot graphics [includegraphics cmd=\pgfimage,xmin=-9.57973134926437, xmax=0.157237353875922, ymin=-17.5771592392966, ymax=1.92424271875102] {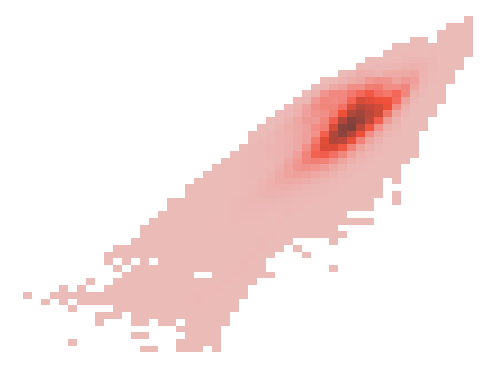};
\end{axis}
\end{scope}

\end{tikzpicture}

%% file: figs/vp_bounds.tex
% This file was created with tikzplotlib v0.10.1.
\begin{tikzpicture}

\definecolor{dimgray85}{RGB}{85,85,85}
\definecolor{gainsboro229}{RGB}{229,229,229}

\begin{axis}[
width=0.92\columnwidth,
height=0.6\columnwidth,
axis background/.style={fill=gainsboro229},
axis line style={white},
scaled ticks=false,
tick align=outside,
tick pos=left,
x grid style={white},
xmajorgrids,
xmin=-17958.4, xmax=377126.4,
legend cell align={left},
legend style={at={(0.03,0.8)},anchor=west},
xtick style={color=dimgray85},
y grid style={white},
ymajorgrids,
ymin=0.993781070602486, ymax=1.03485197053818,
ymin=0.999,ymax=1.001,
ytick style={color=dimgray85},
yticklabel style={/pgf/number format/.cd,fixed,precision=4},
]
\addplot [semithick, red]
table {%
0 1.00000000001047
500 1.00000000875762
1000 1.00000001875047
1500 1.00000002954652
2000 1.00000004086915
2500 1.00000005264041
3000 1.00000006543525
3500 1.00000007938901
4000 1.00000009296004
4500 1.00000010713949
5000 1.0000001224313
5500 1.00000013817657
6000 1.00000015508026
6500 1.00000017383423
7000 1.00000019094017
7500 1.00000021001552
8000 1.00000023002761
8500 1.00000024969304
9000 1.00000026934449
9500 1.00000028933324
10000 1.00000030995589
10500 1.00000033264232
11000 1.00000035590293
11500 1.00000037994818
12000 1.0000004040522
12500 1.00000042752239
13000 1.00000045100179
13500 1.0000004758857
14000 1.00000049958255
14500 1.000000524202
15000 1.0000005507489
15500 1.00000057949789
16000 1.00000060583607
16500 1.00000063317134
17000 1.00000065941082
17500 1.00000068868653
18000 1.00000071648673
18500 1.00000074430843
19000 1.00000077317958
19500 1.00000080280603
20000 1.00000083349227
20500 1.00000086400339
21000 1.00000089423531
21500 1.00000092752977
22000 1.00000095727119
22500 1.00000099014026
23000 1.00000101959873
23500 1.00000105362638
24000 1.00000108589662
24500 1.00000111995264
25000 1.00000115577974
25500 1.00000119004497
26000 1.00000122544703
26500 1.00000126201934
27000 1.00000129613795
27500 1.00000133187527
28000 1.00000136648971
28500 1.00000140424405
29000 1.00000143960098
29500 1.00000147460351
30000 1.00000151375133
30500 1.00000155236966
31000 1.00000159008088
31500 1.00000162925915
32000 1.00000166999431
32500 1.00000170954996
33000 1.00000174863446
33500 1.0000017885152
34000 1.00000182864042
34500 1.00000187105356
35000 1.00000191431186
35500 1.00000195792766
36000 1.00000200220786
36500 1.00000204586646
37000 1.00000209067007
37500 1.00000213590094
38000 1.00000217709129
38500 1.0000022245515
39000 1.00000226748093
39500 1.00000231651997
40000 1.00000236002881
40500 1.00000240270779
41000 1.00000245098229
41500 1.00000249771845
42000 1.0000025418964
42500 1.00000258698619
43000 1.00000263309102
43500 1.00000268233109
44000 1.00000273288853
44500 1.0000027817583
45000 1.00000282784907
45500 1.00000287999861
46000 1.00000292932077
46500 1.00000297885465
47000 1.0000030301898
47500 1.00000308154976
48000 1.00000313575175
48500 1.00000318960803
49000 1.00000323964796
49500 1.00000329805154
50000 1.00000334791569
50500 1.00000340167167
51000 1.00000345199314
51500 1.00000350328994
52000 1.00000355649759
52500 1.00000360941403
53000 1.00000366173451
53500 1.00000371541572
54000 1.0000037694766
54500 1.00000382354483
55000 1.00000387978949
55500 1.00000393764312
56000 1.00000399887633
56500 1.00000405388834
57000 1.00000410780394
57500 1.00000416567512
58000 1.00000421817283
58500 1.00000427620814
59000 1.00000433321902
59500 1.00000438904559
60000 1.00000444456259
60500 1.00000450982869
61000 1.00000456395156
61500 1.00000462586299
62000 1.00000468433964
62500 1.00000474432869
63000 1.00000480906576
63500 1.00000486874699
64000 1.00000492985483
64500 1.00000498821222
65000 1.00000504940425
65500 1.00000511295374
66000 1.00000517888695
66500 1.0000052436318
67000 1.00000530436499
67500 1.00000536640018
68000 1.00000542718181
68500 1.00000549240335
69000 1.00000555122191
69500 1.00000561323146
70000 1.00000568029883
70500 1.00000574907588
71000 1.00000580782518
71500 1.00000587599009
72000 1.00000594221994
72500 1.00000601271026
73000 1.00000607719977
73500 1.00000613858572
74000 1.00000620614336
74500 1.00000627149312
75000 1.00000633662235
75500 1.00000640704649
76000 1.00000647397165
76500 1.00000653991439
77000 1.00000660640916
77500 1.0000066753674
78000 1.00000675130597
78500 1.00000682311043
79000 1.00000689445693
79500 1.00000695843915
80000 1.0000070308915
80500 1.00000710006095
81000 1.0000071763229
81500 1.00000725289767
82000 1.00000732663139
82500 1.00000739898772
83000 1.00000747191303
83500 1.00000754548276
84000 1.00000761969131
84500 1.00000768841897
85000 1.0000077654365
85500 1.00000783996398
86000 1.00000791812624
86500 1.0000079930921
87000 1.00000806899225
87500 1.0000081435494
88000 1.00000822150798
88500 1.00000829254871
89000 1.00000836587061
89500 1.00000844631084
90000 1.00000852100928
90500 1.00000860820409
91000 1.0000086884565
91500 1.00000876787717
92000 1.00000885075023
92500 1.00000893236057
93000 1.00000901820911
93500 1.00000909942824
94000 1.00000917962915
94500 1.0000092612276
95000 1.00000934396264
95500 1.00000942789299
96000 1.00000951338071
96500 1.00000959248339
97000 1.00000967744104
97500 1.00000976424113
98000 1.00000985239604
98500 1.00000993928218
99000 1.00001002112209
99500 1.00001011431526
100000 1.00001019613122
100500 1.00001028597178
101000 1.00001037211418
101500 1.00001045858915
102000 1.00001054868267
102500 1.00001064103926
103000 1.00001073196114
103500 1.00001081779307
104000 1.00001091057214
104500 1.00001099766549
105000 1.00001109178005
105500 1.00001118326262
106000 1.00001127546997
106500 1.0000113690249
107000 1.00001147067438
107500 1.00001156270517
108000 1.00001166250837
108500 1.00001175568947
109000 1.00001185252337
109500 1.00001195084696
110000 1.00001204611093
110500 1.00001214394401
111000 1.00001223917821
111500 1.00001233202729
112000 1.00001243680945
112500 1.00001253655389
113000 1.00001263539452
113500 1.00001273024935
114000 1.00001282711391
114500 1.00001292034427
115000 1.00001301942268
115500 1.00001312131535
116000 1.00001321738139
116500 1.00001331754361
117000 1.00001342084209
117500 1.00001352468698
118000 1.0000136234076
118500 1.00001373512483
119000 1.000013845678
119500 1.00001395523509
120000 1.00001406175729
120500 1.0000141726442
121000 1.0000142820191
121500 1.00001438161507
122000 1.0000144919248
122500 1.00001460090786
123000 1.00001471412894
123500 1.00001482345526
124000 1.00001493607387
124500 1.0000150466896
125000 1.00001516003277
125500 1.00001527002394
126000 1.0000153902813
126500 1.00001550685117
127000 1.00001562518112
127500 1.00001573903041
128000 1.00001584463136
128500 1.00001595745717
129000 1.000016072406
129500 1.00001618786902
130000 1.00001630500188
130500 1.0000164218709
131000 1.00001654425463
131500 1.00001665524708
132000 1.00001676731097
132500 1.00001689276532
133000 1.00001702105364
133500 1.00001714210918
134000 1.00001725966598
134500 1.00001737886405
135000 1.00001749862311
135500 1.00001761742346
136000 1.0000177483731
136500 1.00001787375237
137000 1.00001799762567
137500 1.00001811693115
138000 1.00001823357755
138500 1.00001835445385
139000 1.00001847293174
139500 1.0000186102959
140000 1.00001873286017
140500 1.00001887195247
141000 1.00001899857763
141500 1.00001912316027
142000 1.0000192597631
142500 1.00001938978064
143000 1.00001951659322
143500 1.00001965000608
144000 1.00001977303475
144500 1.00001990816673
145000 1.00002005147723
145500 1.0000201760035
146000 1.00002031157759
146500 1.00002043730158
147000 1.00002057401894
147500 1.00002071397496
148000 1.00002085876454
148500 1.00002101021672
149000 1.00002115521389
149500 1.00002128758236
150000 1.00002142388957
150500 1.00002155699944
151000 1.00002168881813
151500 1.00002182676539
152000 1.00002196120654
152500 1.00002210739014
153000 1.00002225223351
153500 1.00002239599392
154000 1.00002255430071
154500 1.00002269673077
155000 1.00002284557606
155500 1.00002299759484
156000 1.00002314305304
156500 1.00002329444448
157000 1.00002345376939
157500 1.00002361357172
158000 1.00002377284113
158500 1.00002392055202
159000 1.00002408222717
159500 1.00002424916967
160000 1.00002441890035
160500 1.00002457443834
161000 1.00002473691747
161500 1.00002490061257
162000 1.00002506102071
162500 1.00002521680799
163000 1.00002538052721
163500 1.00002553260738
164000 1.00002570867919
164500 1.00002587308331
165000 1.0000260273158
165500 1.00002618747654
166000 1.00002636003084
166500 1.00002653599744
167000 1.00002670930907
167500 1.00002688519822
168000 1.00002705412256
168500 1.00002721658088
169000 1.00002739488511
169500 1.00002756956805
170000 1.00002774684156
170500 1.00002793241008
171000 1.00002809208185
171500 1.00002826769239
172000 1.00002845389671
172500 1.00002865093108
173000 1.0000288237795
173500 1.00002900732718
174000 1.00002919867296
174500 1.00002937352144
175000 1.00002955625218
175500 1.00002974113113
176000 1.00002992297581
176500 1.00003011691454
177000 1.00003030567265
177500 1.00003049123614
178000 1.00003067990872
178500 1.00003087178895
179000 1.00003105475037
179500 1.00003124299816
180000 1.00003143903872
180500 1.00003162465245
181000 1.00003183075454
181500 1.00003203259773
182000 1.00003224437626
182500 1.00003245676629
183000 1.00003265483711
183500 1.00003285243345
184000 1.00003305609194
184500 1.0000332653135
185000 1.00003347656285
185500 1.00003367682978
186000 1.00003388851879
186500 1.00003407626697
187000 1.00003426994206
187500 1.00003448284009
188000 1.00003469901613
188500 1.00003492027571
189000 1.00003514291289
189500 1.00003536734234
190000 1.0000355900869
190500 1.00003581129576
191000 1.00003601659911
191500 1.00003622622949
192000 1.00003646401646
192500 1.00003668524014
193000 1.00003692950367
193500 1.00003717525558
194000 1.00003741205665
194500 1.00003765571031
195000 1.00003787429209
195500 1.00003811137204
196000 1.00003835233689
196500 1.00003858663333
197000 1.00003882264474
197500 1.00003905760121
198000 1.00003930251507
198500 1.00003953984992
199000 1.00003977866281
199500 1.00004002113275
200000 1.00004026949714
200500 1.00004053694593
201000 1.00004078990826
201500 1.00004103629163
202000 1.00004128339188
202500 1.00004154812566
203000 1.00004181999572
203500 1.00004207575066
204000 1.00004232097779
204500 1.00004259255173
205000 1.00004285654394
205500 1.00004311739854
206000 1.00004338767719
206500 1.00004366935442
207000 1.00004394368798
207500 1.00004422840604
208000 1.00004449027416
208500 1.00004476636391
209000 1.00004506168841
209500 1.00004535019334
210000 1.00004563980437
210500 1.00004590513847
211000 1.00004616519948
211500 1.00004644958061
212000 1.00004676332173
212500 1.0000470473438
213000 1.00004734588625
213500 1.00004765589502
214000 1.00004795587358
214500 1.00004825626174
215000 1.00004855892374
215500 1.00004885471496
216000 1.00004915947099
216500 1.00004947414504
217000 1.00004980947487
217500 1.00005013313496
218000 1.0000504717893
218500 1.00005079961828
219000 1.00005112910327
219500 1.00005143880878
220000 1.0000517532593
220500 1.00005207077112
221000 1.00005241552128
221500 1.00005274247824
222000 1.00005308901126
222500 1.00005341085835
223000 1.00005375141161
223500 1.00005408161227
224000 1.00005441723789
224500 1.00005478322847
225000 1.00005513830183
225500 1.00005549327515
226000 1.00005585875009
226500 1.0000562029288
227000 1.00005657423954
227500 1.00005695460447
228000 1.00005731269733
228500 1.0000576901555
229000 1.0000580786336
229500 1.00005843670422
230000 1.00005882415517
230500 1.00005919487043
231000 1.00005957644108
231500 1.00005993740859
232000 1.00006034408561
232500 1.0000607385784
233000 1.00006114413526
233500 1.00006153697148
234000 1.00006194950326
234500 1.00006239107679
235000 1.00006281422331
235500 1.00006328694437
236000 1.00006370139052
236500 1.00006413697991
237000 1.00006460759719
237500 1.00006506278653
238000 1.0000654795912
238500 1.00006592222335
239000 1.0000663915399
239500 1.00006683946538
240000 1.00006730214652
240500 1.0000677516837
241000 1.00006820233313
241500 1.00006866963488
242000 1.00006908738569
242500 1.00006951404796
243000 1.00006998064525
243500 1.00007045331373
244000 1.000070953632
244500 1.00007140923667
245000 1.00007191496532
245500 1.00007239533642
246000 1.00007289316986
246500 1.00007336424901
247000 1.00007385417547
247500 1.00007437006677
248000 1.00007489429742
248500 1.00007542937947
249000 1.00007592972869
249500 1.00007648796656
250000 1.00007704903344
250500 1.00007757786446
251000 1.00007808959578
251500 1.0000786279159
252000 1.00007921212256
252500 1.00007975239055
253000 1.00008026595148
253500 1.00008088501104
254000 1.00008141194004
254500 1.00008199806159
255000 1.00008258592969
255500 1.00008316928593
256000 1.00008376762655
256500 1.00008442732946
257000 1.0000850577986
257500 1.00008565354284
258000 1.00008632125802
258500 1.0000869677431
259000 1.00008758634793
259500 1.00008820987921
260000 1.00008885573142
260500 1.00008951540292
261000 1.00009015195414
261500 1.00009080740181
262000 1.00009147313225
262500 1.00009211517266
263000 1.00009278267397
263500 1.00009346200131
264000 1.00009411821805
264500 1.00009479904079
265000 1.00009551855189
265500 1.00009620966921
266000 1.00009690341622
266500 1.00009766610594
267000 1.00009847214316
267500 1.00009923434621
268000 1.00009995273403
268500 1.00010067299717
269000 1.000101440607
269500 1.00010223275675
270000 1.00010298990377
270500 1.00010376325938
271000 1.00010456271489
271500 1.00010537300923
272000 1.00010620095383
272500 1.0001070375485
273000 1.00010782865572
273500 1.00010866088941
274000 1.00010955722706
274500 1.00011037511607
275000 1.00011121079763
275500 1.00011212976902
276000 1.00011302080671
276500 1.00011386799358
277000 1.00011474663649
277500 1.00011567273855
278000 1.0001166356158
278500 1.00011756134265
279000 1.00011850905563
279500 1.00011946352742
280000 1.00012035792708
280500 1.00012131554866
281000 1.00012230645044
281500 1.00012333204219
282000 1.00012448058127
282500 1.00012546772157
283000 1.00012647247428
283500 1.00012752410741
284000 1.00012856413105
284500 1.00012971959974
285000 1.00013080626697
285500 1.00013189719366
286000 1.00013294342063
286500 1.00013413722095
287000 1.00013528792648
287500 1.00013642130734
288000 1.00013752248374
288500 1.00013864246642
289000 1.00013980439937
289500 1.00014094894136
290000 1.00014220395305
290500 1.00014348736162
291000 1.00014480626861
291500 1.00014611630185
292000 1.00014740591533
292500 1.00014873739436
293000 1.00015006369534
293500 1.00015137770497
294000 1.00015285978479
294500 1.00015427473046
295000 1.00015562734003
295500 1.00015705447405
296000 1.00015845392858
296500 1.00016000350388
297000 1.00016151583345
297500 1.00016308425461
298000 1.0001645023316
298500 1.00016604610561
299000 1.0001675832409
299500 1.00016903407163
300000 1.00017068653228
300500 1.00017231568
301000 1.00017395487691
301500 1.00017562516383
302000 1.00017742253445
302500 1.00017929175291
303000 1.00018112163267
303500 1.00018295543123
304000 1.00018489159114
304500 1.0001867379869
305000 1.00018861677633
305500 1.0001905868722
306000 1.00019247663633
306500 1.00019457570567
307000 1.00019664932686
307500 1.00019865379176
308000 1.00020076648081
308500 1.00020289382948
309000 1.00020501830704
309500 1.00020710333546
310000 1.00020939190083
310500 1.00021159531197
311000 1.00021386862477
311500 1.00021642434623
312000 1.00021871004606
312500 1.00022111064601
313000 1.00022368389222
313500 1.00022622484999
314000 1.00022890149615
314500 1.00023166644301
315000 1.00023394710938
315500 1.0002365756295
316000 1.00023926748262
316500 1.00024189075986
317000 1.00024460819929
317500 1.00024753084324
318000 1.00025053860324
318500 1.00025375932956
319000 1.00025667482289
319500 1.00025982946243
320000 1.00026307548247
320500 1.00026612745683
321000 1.00026959955809
321500 1.00027302638183
322000 1.00027609290585
322500 1.00027941047333
323000 1.00028295446394
323500 1.00028674239276
324000 1.00029055555016
324500 1.00029422021733
325000 1.00029830744567
325500 1.00030209893574
326000 1.0003059998714
326500 1.00031025603886
327000 1.00031474526057
327500 1.00031936477492
328000 1.00032368997943
328500 1.00032859034939
329000 1.00033328643837
329500 1.00033819422306
330000 1.00034331430667
330500 1.00034822416318
331000 1.00035349550124
331500 1.00035861981652
332000 1.00036471435906
332500 1.0003703257011
333000 1.0003759178847
333500 1.00038161660314
334000 1.00038774077238
334500 1.00039408522147
335000 1.00040035990575
335500 1.00040743068891
336000 1.0004144269142
336500 1.0004215641797
337000 1.00042908085212
337500 1.00043747233983
338000 1.00044585264439
338500 1.00045389064179
339000 1.00046218291887
339500 1.00047081991225
340000 1.00047971040075
340500 1.00048930287927
341000 1.00049957654233
341500 1.00051045807799
342000 1.00052073729653
342500 1.00053141209413
343000 1.0005434006571
343500 1.0005558701785
344000 1.00056793185378
344500 1.0005809241664
345000 1.00059490344666
345500 1.00060860740082
346000 1.000624655924
346500 1.00064164151948
347000 1.00065853663751
347500 1.00067580344467
348000 1.00069556138319
348500 1.00071733804208
349000 1.00073867645243
349500 1.00076173748838
350000 1.00078931036073
350500 1.00081712713261
351000 1.00084760508561
351500 1.00087930430907
352000 1.0009138909168
352500 1.00095185748465
353000 1.00099506253833
353500 1.00103815875921
354000 1.00108832691519
354500 1.00115379063894
355000 1.00122496596927
355500 1.00130880395539
356000 1.00140770468415
356500 1.00152747353656
357000 1.00167558481901
357500 1.00186920285263
358000 1.00219998934859
358500 1.00284076436404
359000 1.00487932995023
359069 1.00562665330397
359070 1.00566780495682
359071 1.00567268727665
359072 1.00571632927049
359073 1.00574769330172
359074 1.00576990497786
359075 1.00579065372665
359076 1.00579207209269
359077 1.00579349237623
359078 1.00581432635265
359079 1.00581443418197
359080 1.00584438233472
359081 1.00584981851049
359082 1.00585230387963
359083 1.00590797332868
359084 1.00601841821983
359085 1.00603773347151
359086 1.00605063758931
359087 1.00607111197283
359088 1.00609109225462
359089 1.00611521755689
359090 1.00619045591804
359091 1.00620276881433
359092 1.00620606845257
359093 1.00621185167538
359094 1.0062138168571
359095 1.00621898854388
359096 1.00622715709145
359097 1.00624403414673
359098 1.00625352417064
359099 1.00627389757482
359100 1.00630113610904
359101 1.00630154657054
359102 1.00641786049249
359103 1.00646991902308
359104 1.00652723109863
359105 1.00654457927223
359106 1.00657538812832
359107 1.00658499871849
359108 1.00659073021731
359109 1.00660886992336
359110 1.00661873415776
359111 1.00665520042374
359112 1.00670371930964
359113 1.00687011614817
359114 1.00688532606086
359115 1.00690245279913
359116 1.00696401273285
359117 1.00698019636465
359118 1.00698240200925
359119 1.00700750763928
359120 1.00706503411107
359121 1.00707825765411
359122 1.00708984242255
359123 1.00712882484291
359124 1.00713289052244
359125 1.00719858902381
359126 1.00728002980921
359127 1.00736687760885
359128 1.0075609866617
359129 1.00758133740207
359130 1.00765412756532
359131 1.00768968168889
359132 1.00772057062912
359133 1.00778391835982
359134 1.00778670362369
359135 1.00779316652045
359136 1.007802293698
359137 1.00794874936026
359138 1.00796937599877
359139 1.0080521348973
359140 1.00830300170013
359141 1.00843800536895
359142 1.0084923260263
359143 1.0084977967024
359144 1.00854904313788
359145 1.00861645915509
359146 1.0088033856209
359147 1.00891568865573
359148 1.00912251022427
359149 1.0092782011013
359150 1.00929931524827
359151 1.00936566778138
359152 1.00983340589144
359153 1.00999016447407
359154 1.0101803418198
359155 1.01036290568867
359156 1.01087993816005
359157 1.01092285273793
359158 1.01115016054114
359159 1.0117812365572
359160 1.01232762750416
359161 1.01266770589334
359162 1.01275473250689
359163 1.01339080006764
359164 1.01372057411342
359165 1.0148785331165
359166 1.01803008190511
359167 1.03252610618233
359168 1.03298511145019
};
\addlegendentry{$B_u = 1 + \frac{\rho}{2\|J\|}\|\ee^G\|$}

\addplot [semithick, blue]
table {%
0 1
500 1
1000 0.999999999999998
1500 0.999999999999996
2000 0.999999999999993
2500 0.999999999999989
3000 0.999999999999983
3500 0.999999999999975
4000 0.999999999999965
4500 0.999999999999954
5000 0.99999999999994
5500 0.999999999999924
6000 0.999999999999904
6500 0.999999999999879
7000 0.999999999999854
7500 0.999999999999824
8000 0.999999999999788
8500 0.999999999999751
9000 0.99999999999971
9500 0.999999999999665
10000 0.999999999999616
10500 0.999999999999557
11000 0.999999999999493
11500 0.999999999999422
12000 0.999999999999347
12500 0.999999999999269
13000 0.999999999999186
13500 0.999999999999094
14000 0.999999999999002
14500 0.999999999998901
15000 0.999999999998787
15500 0.999999999998657
16000 0.999999999998532
16500 0.999999999998396
17000 0.999999999998261
17500 0.999999999998103
18000 0.999999999997947
18500 0.999999999997784
19000 0.999999999997609
19500 0.999999999997422
20000 0.999999999997221
20500 0.999999999997014
21000 0.999999999996801
21500 0.999999999996559
22000 0.999999999996334
22500 0.999999999996078
23000 0.999999999995842
23500 0.99999999999556
24000 0.999999999995283
24500 0.999999999994983
25000 0.999999999994657
25500 0.999999999994335
26000 0.999999999993993
26500 0.999999999993629
27000 0.99999999999328
27500 0.999999999992904
28000 0.999999999992531
28500 0.999999999992113
29000 0.99999999999171
29500 0.999999999991302
30000 0.999999999990834
30500 0.999999999990361
31000 0.999999999989887
31500 0.999999999989382
32000 0.999999999988844
32500 0.99999999998831
33000 0.999999999987769
33500 0.999999999987205
34000 0.999999999986624
34500 0.999999999985997
35000 0.999999999985342
35500 0.999999999984666
36000 0.999999999983965
36500 0.999999999983258
37000 0.999999999982516
37500 0.999999999981751
38000 0.999999999981041
38500 0.999999999980205
39000 0.999999999979434
39500 0.999999999978535
40000 0.999999999977721
40500 0.999999999976908
41000 0.999999999975971
41500 0.999999999975046
42000 0.999999999974155
42500 0.99999999997323
43000 0.999999999972267
43500 0.999999999971221
44000 0.999999999970125
44500 0.999999999969047
45000 0.999999999968013
45500 0.999999999966822
46000 0.999999999965677
46500 0.999999999964506
47000 0.999999999963272
47500 0.999999999962016
48000 0.999999999960668
48500 0.999999999959306
49000 0.999999999958019
49500 0.999999999956492
50000 0.999999999955166
50500 0.999999999953715
51000 0.999999999952335
51500 0.999999999950908
52000 0.999999999949405
52500 0.999999999947889
53000 0.999999999946367
53500 0.999999999944783
54000 0.999999999943164
54500 0.999999999941522
55000 0.999999999939789
55500 0.99999999993798
56000 0.999999999936036
56500 0.999999999934264
57000 0.999999999932504
57500 0.999999999930589
58000 0.999999999928828
58500 0.999999999926856
59000 0.999999999924893
59500 0.999999999922945
60000 0.999999999920983
60500 0.999999999918646
61000 0.999999999916681
61500 0.999999999914406
62000 0.999999999912228
62500 0.999999999909966
63000 0.999999999907492
63500 0.999999999905181
64000 0.999999999902786
64500 0.999999999900471
65000 0.999999999898014
65500 0.999999999895431
66000 0.999999999892716
66500 0.999999999890017
67000 0.999999999887455
67500 0.999999999884807
68000 0.999999999882183
68500 0.999999999879334
69000 0.999999999876736
69500 0.999999999873966
70000 0.999999999870937
70500 0.999999999867792
71000 0.999999999865076
71500 0.999999999861891
72000 0.99999999985876
72500 0.999999999855389
73000 0.999999999852271
73500 0.999999999849271
74000 0.999999999845935
74500 0.999999999842673
75000 0.999999999839389
75500 0.999999999835799
76000 0.999999999832351
76500 0.999999999828918
77000 0.999999999825421
77500 0.999999999821758
78000 0.99999999981768
78500 0.999999999813781
79000 0.999999999809866
79500 0.99999999980632
80000 0.999999999802266
80500 0.999999999798357
81000 0.999999999794002
81500 0.999999999789582
82000 0.999999999785282
82500 0.99999999978102
83000 0.999999999776682
83500 0.999999999772263
84000 0.999999999767761
84500 0.999999999763553
85000 0.999999999758792
85500 0.99999999975414
86000 0.999999999749213
86500 0.999999999744442
87000 0.999999999739565
87500 0.99999999973473
88000 0.999999999729627
88500 0.999999999724934
89000 0.999999999720049
89500 0.999999999714639
90000 0.999999999709569
90500 0.999999999703595
91000 0.999999999698043
91500 0.999999999692497
92000 0.999999999686657
92500 0.999999999680852
93000 0.999999999674687
93500 0.999999999668802
94000 0.999999999662938
94500 0.999999999656919
95000 0.999999999650761
95500 0.999999999644459
96000 0.999999999637983
96500 0.999999999631937
97000 0.999999999625389
97500 0.999999999618638
98000 0.999999999611721
98500 0.999999999604843
99000 0.999999999598308
99500 0.999999999590803
100000 0.999999999584156
100500 0.999999999576795
101000 0.999999999569677
101500 0.999999999562472
102000 0.999999999554901
102500 0.999999999547073
103000 0.9999999995393
103500 0.999999999531901
104000 0.999999999523838
104500 0.999999999516205
105000 0.99999999950789
105500 0.999999999499739
106000 0.999999999491455
106500 0.999999999482981
107000 0.999999999473695
107500 0.999999999465215
108000 0.999999999455944
108500 0.999999999447215
109000 0.999999999438071
109500 0.999999999428709
110000 0.999999999419565
110500 0.999999999410099
111000 0.99999999940081
111500 0.999999999391684
112000 0.999999999381303
112500 0.999999999371339
113000 0.999999999361387
113500 0.999999999351763
114000 0.999999999341861
114500 0.999999999332259
115000 0.999999999321979
115500 0.999999999311324
116000 0.999999999301203
116500 0.999999999290572
117000 0.999999999279524
117500 0.999999999268331
118000 0.999999999257611
118500 0.999999999245386
119000 0.999999999233189
119500 0.999999999221006
120000 0.999999999209068
120500 0.999999999196544
121000 0.999999999184096
121500 0.999999999172676
122000 0.999999999159936
122500 0.999999999147254
123000 0.999999999133977
123500 0.999999999121061
124000 0.999999999107655
124500 0.999999999094388
125000 0.999999999080694
125500 0.999999999067305
126000 0.999999999052557
126500 0.99999999903815
127000 0.999999999023415
127500 0.999999999009132
128000 0.999999998995791
128500 0.999999998981438
129000 0.999999998966711
129500 0.999999998951812
130000 0.999999998936588
130500 0.999999998921289
131000 0.99999999890515
131500 0.999999998890411
132000 0.999999998875429
132500 0.999999998858538
133000 0.999999998841135
133500 0.999999998824592
134000 0.999999998808416
134500 0.9999999987919
135000 0.999999998775193
135500 0.999999998758505
136000 0.999999998739981
136500 0.999999998722116
137000 0.999999998704342
137500 0.999999998687107
138000 0.999999998670147
138500 0.999999998652456
139000 0.999999998635003
139500 0.999999998614628
140000 0.99999999859632
140500 0.999999998575398
141000 0.999999998556216
141500 0.999999998537219
142000 0.999999998516246
142500 0.999999998496146
143000 0.99999999847641
143500 0.999999998455509
144000 0.999999998436109
144500 0.999999998414659
145000 0.999999998391753
145500 0.999999998371715
146000 0.999999998349759
146500 0.999999998329267
147000 0.999999998306839
147500 0.999999998283725
148000 0.999999998259648
148500 0.999999998234283
149000 0.999999998209828
149500 0.999999998187355
150000 0.999999998164068
150500 0.999999998141183
151000 0.999999998118381
151500 0.999999998094369
152000 0.999999998070822
152500 0.999999998045053
153000 0.999999998019353
153500 0.999999997993678
154000 0.999999997965214
154500 0.999999997939434
155000 0.999999997912319
155500 0.999999997884443
156000 0.999999997857596
156500 0.999999997829475
157000 0.999999997799683
157500 0.999999997769597
158000 0.999999997739408
158500 0.999999997711229
159000 0.999999997680185
159500 0.999999997647911
160000 0.999999997614869
160500 0.999999997584388
161000 0.99999999755234
161500 0.999999997519838
162000 0.999999997487781
162500 0.99999999745645
163000 0.999999997423315
163500 0.999999997392344
164000 0.999999997356255
164500 0.999999997322334
165000 0.999999997290315
165500 0.999999997256864
166000 0.999999997220595
166500 0.999999997183364
167000 0.999999997146451
167500 0.999999997108745
168000 0.999999997072298
168500 0.999999997037031
169000 0.999999996998081
169500 0.999999996959676
170000 0.999999996920451
170500 0.999999996879122
171000 0.99999999684334
171500 0.99999999680375
172000 0.999999996761503
172500 0.999999996716497
173000 0.999999996676759
173500 0.9999999966343
174000 0.99999999658975
174500 0.999999996548785
175000 0.999999996505712
175500 0.99999999646186
176000 0.999999996418462
176500 0.999999996371886
177000 0.999999996326265
177500 0.999999996281138
178000 0.999999996234973
178500 0.999999996187731
179000 0.99999999614241
179500 0.9999999960955
180000 0.999999996046347
180500 0.999999995999525
181000 0.999999995947212
181500 0.999999995895651
182000 0.999999995841201
182500 0.999999995786233
183000 0.999999995734646
183500 0.999999995682871
184000 0.999999995629179
184500 0.999999995573676
185000 0.999999995517279
185500 0.999999995463485
186000 0.999999995406273
186500 0.999999995355232
187000 0.999999995302284
187500 0.999999995243735
188000 0.999999995183913
188500 0.999999995122297
189000 0.999999995059903
189500 0.999999994996604
190000 0.999999994933383
190500 0.999999994870204
191000 0.999999994811218
191500 0.999999994750641
192000 0.999999994681502
192500 0.999999994616773
193000 0.999999994544847
193500 0.999999994472001
194000 0.999999994401352
194500 0.99999999432819
195000 0.999999994262152
195500 0.999999994190093
196000 0.999999994116393
196500 0.999999994044287
197000 0.999999993971209
197500 0.999999993898015
198000 0.999999993821249
198500 0.999999993746401
199000 0.999999993670632
199500 0.999999993593236
200000 0.99999999351347
200500 0.999999993427024
201000 0.999999993344734
201500 0.999999993264091
202000 0.999999993182726
202500 0.999999993095013
203000 0.999999993004352
203500 0.999999992918525
204000 0.999999992835739
204500 0.999999992743498
205000 0.999999992653267
205500 0.99999999256356
206000 0.999999992470038
206500 0.99999999237195
207000 0.999999992275809
207500 0.999999992175392
208000 0.999999992082462
208500 0.999999991983891
209000 0.999999991877777
209500 0.99999999177344
210000 0.999999991668033
210500 0.999999991570873
211000 0.999999991475097
211500 0.999999991369746
212000 0.999999991252767
212500 0.99999999114619
213000 0.999999991033468
213500 0.999999990915663
214000 0.999999990800937
214500 0.999999990685333
215000 0.999999990568123
215500 0.999999990452867
216000 0.999999990333386
216500 0.999999990209236
217000 0.999999990076065
217500 0.999999989946675
218000 0.999999989810394
218500 0.999999989677595
219000 0.999999989543259
219500 0.999999989416196
220000 0.999999989286401
220500 0.999999989154539
221000 0.999999989010452
221500 0.999999988872924
222000 0.999999988726228
222500 0.999999988589121
223000 0.999999988443143
223500 0.999999988300717
224000 0.999999988155057
224500 0.999999987995192
225000 0.999999987839071
225500 0.999999987681986
226000 0.9999999875192
226500 0.999999987364923
227000 0.999999987197422
227500 0.999999987024692
228000 0.999999986861019
228500 0.999999986687384
229000 0.999999986507489
229500 0.999999986340606
230000 0.999999986158875
230500 0.999999985983869
231000 0.999999985802591
231500 0.999999985630028
232000 0.999999985434365
232500 0.9999999852433
233000 0.999999985045579
233500 0.999999984852805
234000 0.999999984649036
234500 0.999999984429414
235000 0.999999984217494
235500 0.999999983979051
236000 0.999999983768531
236500 0.999999983545791
237000 0.999999983303434
237500 0.999999983067335
238000 0.999999982849692
238500 0.999999982617042
239000 0.999999982368654
239500 0.999999982129943
240000 0.999999981881684
240500 0.999999981638837
241000 0.999999981393767
241500 0.999999981137925
242000 0.999999980907733
242500 0.999999980671189
243000 0.999999980410837
243500 0.999999980145323
244000 0.999999979862328
244500 0.999999979602884
245000 0.999999979312951
245500 0.999999979035661
246000 0.999999978746343
246500 0.999999978470748
247000 0.999999978182243
247500 0.999999977876373
248000 0.999999977563377
248500 0.999999977241635
249000 0.999999976938705
249500 0.999999976598364
250000 0.999999976253786
250500 0.9999999759267
251000 0.99999997560806
251500 0.999999975270603
252000 0.999999974901759
252500 0.999999974558225
253000 0.999999974229508
253500 0.99999997383046
254000 0.999999973488384
254500 0.999999973105272
255000 0.999999972718257
255500 0.99999997233148
256000 0.999999971931939
256500 0.999999971488104
257000 0.999999971060684
257500 0.999999970653882
258000 0.999999970194562
258500 0.999999969746447
259000 0.999999969314527
259500 0.999999968876069
260000 0.999999968418636
260500 0.999999967947971
261000 0.9999999674905
261500 0.999999967016063
262000 0.999999966530664
262500 0.99999996605918
263000 0.999999965565502
263500 0.999999965059417
264000 0.999999964567044
264500 0.999999964052567
265000 0.999999963504825
265500 0.999999962974798
266000 0.999999962438912
266500 0.999999961845327
267000 0.999999961212948
267500 0.999999960610178
268000 0.999999960037804
268500 0.99999995945979
269000 0.999999958839213
269500 0.999999958193854
270000 0.999999957572319
270500 0.999999956932744
271000 0.999999956266555
271500 0.999999955586116
272000 0.99999995488543
272500 0.999999954171853
273000 0.999999953491924
273500 0.999999952771244
274000 0.999999951988856
274500 0.999999951269335
275000 0.999999950528634
275500 0.99999994970766
276000 0.999999948905189
276500 0.99999994813632
277000 0.999999947332838
277500 0.99999994647927
278000 0.999999945584533
278500 0.999999944717323
279000 0.999999943822415
279500 0.999999942913862
280000 0.999999942055877
280500 0.999999941130151
281000 0.999999940164529
281500 0.999999939156829
282000 0.99999993801834
282500 0.999999937031403
283000 0.999999936018853
283500 0.999999934950408
284000 0.999999933885057
284500 0.999999932691302
285000 0.999999931558882
285500 0.999999930412521
286000 0.999999929304188
286500 0.999999928028824
287000 0.999999926788708
287500 0.999999925556908
288000 0.999999924350266
288500 0.999999923113066
289000 0.99999992181892
289500 0.999999920533584
290000 0.999999919112143
290500 0.999999917645508
291000 0.999999916124578
291500 0.999999914600105
292000 0.999999913085984
292500 0.99999991150875
293000 0.999999909923549
293500 0.999999908339162
294000 0.999999906535545
294500 0.99999990479723
295000 0.999999903120524
295500 0.999999901335569
296000 0.99999989956941
296500 0.999999897595515
297000 0.999999895650542
297500 0.999999893614103
298000 0.999999891755932
298500 0.999999889714764
299000 0.999999887663429
299500 0.999999885709931
300000 0.999999883464431
300500 0.999999881229226
301000 0.999999878958803
301500 0.999999876623207
302000 0.999999874084977
302500 0.999999871417869
303000 0.999999868779817
303500 0.999999866109241
304000 0.999999863260398
304500 0.999999860515697
305000 0.999999857694847
305500 0.999999854706577
306000 0.999999851810978
306500 0.999999848561179
307000 0.999999845316169
307500 0.999999842146684
308000 0.999999838771281
308500 0.999999835336376
309000 0.999999831869975
309500 0.999999828432834
310000 0.999999824620127
310500 0.999999820909696
311000 0.999999817040846
311500 0.99999981264201
312000 0.999999808663663
312500 0.999999804440329
313000 0.999999799862065
313500 0.999999795289269
314000 0.99999979041642
314500 0.999999785322637
315000 0.999999781075
315500 0.999999776127886
316000 0.999999771004287
316500 0.999999765955441
317000 0.999999760667315
317500 0.999999754913927
318000 0.999999748921633
318500 0.999999742424811
319000 0.999999736472141
319500 0.999999729954602
320000 0.999999723165162
320500 0.999999716704707
321000 0.999999709264313
321500 0.999999701826379
322000 0.999999695090829
322500 0.99999968771915
323000 0.999999679747085
323500 0.999999671115201
324000 0.999999662309889
324500 0.999999653737855
325000 0.999999644050671
325500 0.999999634944932
326000 0.999999625456315
326500 0.999999614964761
327000 0.999999603741684
327500 0.999999592024562
328000 0.999999580899189
328500 0.999999568113529
329000 0.9999995556806
329500 0.99999954249867
330000 0.999999528541147
330500 0.999999514959729
331000 0.999999500163722
331500 0.999999485567309
332000 0.999999467933745
332500 0.9999994514355
333000 0.999999434742976
333500 0.999999417475073
334000 0.999999398628374
334500 0.999999378787353
335000 0.999999358847783
335500 0.999999336000935
336000 0.999999313001331
336500 0.99999928913457
337000 0.999999263558489
337500 0.999999234471808
338000 0.999999204861678
338500 0.999999175933141
339000 0.999999145547798
339500 0.999999113314441
340000 0.999999079511726
340500 0.999999042330769
341000 0.999999001693113
341500 0.999998957730203
342000 0.999998915330672
342500 0.999998870404745
343000 0.999998818862903
343500 0.999998764033379
344000 0.999998709813638
344500 0.999998650108451
345000 0.999998584359557
345500 0.999998518388127
346000 0.999998439219907
346500 0.999998353184642
347000 0.999998265317988
347500 0.999998173158817
348000 0.999998064777449
348500 0.999997941704533
349000 0.999997817428394
349500 0.999997679023995
350000 0.999997507956618
350500 0.999997329212997
351000 0.999997126262475
351500 0.999996907295728
352000 0.999996659213569
352500 0.999996375869316
353000 0.999996039402179
353500 0.999995688905563
354000 0.999995262178103
354500 0.999994675068646
355000 0.999993997833497
355500 0.999993148128825
356000 0.999992073470089
356500 0.99999066729838
357000 0.999988769662057
357500 0.999986024322783
358000 0.999980640187464
358500 0.999967720231312
359000 0.999904768556947
359069 0.999873363090388
359070 0.999871503947886
359071 0.999871282476245
359072 0.999869294318685
359073 0.999867856086837
359074 0.999866832786186
359075 0.999865873317672
359076 0.999865807603492
359077 0.999865741784346
359078 0.999864774436259
359079 0.999864769420574
359080 0.999863372780502
359081 0.999863118493578
359082 0.999863002157202
359083 0.999860383404591
359084 0.999855114568525
359085 0.999854183098108
359086 0.999853559139051
359087 0.999852566397653
359088 0.999851594380582
359089 0.999850416456928
359090 0.999846713022107
359091 0.999846102636144
359092 0.999845938857448
359093 0.999845651595053
359094 0.999845553920266
359095 0.999845296725964
359096 0.999844890058233
359097 0.999844048150298
359098 0.999843573741789
359099 0.999842552836883
359100 0.999841182734941
359101 0.999841162043277
359102 0.999835244266795
359103 0.999832560591339
359104 0.999829581016741
359105 0.999828673928598
359106 0.999827057083848
359107 0.99982655116751
359108 0.99982624910081
359109 0.999825291353345
359110 0.999824769432596
359111 0.999822833229279
359112 0.99982024058967
359113 0.999811206016442
359114 0.999810369140143
359115 0.999809424581423
359116 0.999806010106627
359117 0.999805107434844
359118 0.999804984248725
359119 0.999803579346742
359120 0.999800341172038
359121 0.999799593074328
359122 0.999798936537693
359123 0.999796719425436
359124 0.999796487491179
359125 0.999792721264265
359126 0.999788004663908
359127 0.999782916457185
359128 0.999771325922807
359129 0.999770093292784
359130 0.999765657324855
359131 0.999763475182094
359132 0.999761571156643
359133 0.99975764245987
359134 0.999757468986708
359135 0.999757066222338
359136 0.9997564968522
359137 0.999747269534431
359138 0.999745956184761
359139 0.999740652494383
359140 0.999724240651071
359141 0.999715200261575
359142 0.999711521594652
359143 0.999711149804819
359144 0.999707655445706
359145 0.999703026526514
359146 0.999690001606439
359147 0.999682041983176
359148 0.999667119228832
359149 0.999655659937295
359150 0.999654090943653
359151 0.999649137068036
359152 0.999613216514297
359153 0.999600786455124
359154 0.999585442561728
359155 0.999570440742751
359156 0.999526507782534
359157 0.999522765152262
359158 0.999502695679627
359159 0.999444809860733
359160 0.999392118400475
359161 0.999358116909599
359162 0.99934926719471
359163 0.999282745894193
359164 0.999246983383993
359165 0.999114517009204
359166 0.99869966458598
359167 0.995768209666464
359168 0.995647929690472
};
\addlegendentry{$B_l = \tau^{-1}$}

\end{axis}
\end{tikzpicture}

%% file: figs/vp_bounds_diff.tex
% This file was created with tikzplotlib v0.10.1.
\begin{tikzpicture}

\definecolor{chocolate2267451}{RGB}{226,74,51}
\definecolor{dimgray85}{RGB}{85,85,85}
\definecolor{gainsboro229}{RGB}{229,229,229}
\definecolor{steelblue52138189}{RGB}{52,138,189}

\begin{axis}[
width=\columnwidth,
height=0.6\columnwidth,
axis background/.style={fill=gainsboro229},
axis line style={white},
tick align=outside,
tick pos=left,
x grid style={white},
xlabel=\textcolor{dimgray85}{$\text{log}_{10}(B_u - B_l)$},
xmajorgrids,
xmin=-11.4578341529629, xmax=-0.950240576190655,
xtick style={color=dimgray85},
y grid style={white},
ymajorgrids,
ymin=0, ymax=0.541120741628624,
ytick style={color=dimgray85}
]
\draw[draw=none,fill=chocolate2267451,fill opacity=0.5, very thin] (axis cs:-10.9802162631096,0) rectangle (axis cs:-10.5025983732563,1.74880676121693e-05);
\draw[draw=none,fill=chocolate2267451,fill opacity=0.5, very thin] (axis cs:-10.5025983732563,0) rectangle (axis cs:-10.024980483403,3.49761352243386e-05);
\draw[draw=none,fill=chocolate2267451,fill opacity=0.5, very thin] (axis cs:-10.024980483403,0) rectangle (axis cs:-9.54736259354974,7.57816263194004e-05);
\draw[draw=none,fill=chocolate2267451,fill opacity=0.5, very thin] (axis cs:-9.54736259354974,0) rectangle (axis cs:-9.06974470369645,0.000262321014182539);
\draw[draw=none,fill=chocolate2267451,fill opacity=0.5, very thin] (axis cs:-9.06974470369645,0) rectangle (axis cs:-8.59212681384317,0.000612082366425924);
\draw[draw=none,fill=chocolate2267451,fill opacity=0.5, very thin] (axis cs:-8.59212681384317,0) rectangle (axis cs:-8.11450892398989,0.00165553706728536);
\draw[draw=none,fill=chocolate2267451,fill opacity=0.5, very thin] (axis cs:-8.11450892398989,0) rectangle (axis cs:-7.63689103413661,0.00430789398846437);
\draw[draw=none,fill=chocolate2267451,fill opacity=0.5, very thin] (axis cs:-7.6368910341366,0) rectangle (axis cs:-7.15927314428332,0.0112098513394005);
\draw[draw=none,fill=chocolate2267451,fill opacity=0.5, very thin] (axis cs:-7.15927314428332,0) rectangle (axis cs:-6.68165525443004,0.0253576980376455);
\draw[draw=none,fill=chocolate2267451,fill opacity=0.5, very thin] (axis cs:-6.68165525443004,0) rectangle (axis cs:-6.20403736457676,0.0518346324024698);
\draw[draw=none,fill=chocolate2267451,fill opacity=0.5, very thin] (axis cs:-6.20403736457676,0) rectangle (axis cs:-5.72641947472348,0.106164229117609);
\draw[draw=none,fill=chocolate2267451,fill opacity=0.5, very thin] (axis cs:-5.72641947472348,0) rectangle (axis cs:-5.2488015848702,0.20484356529721);
\draw[draw=none,fill=chocolate2267451,fill opacity=0.5, very thin] (axis cs:-5.2488015848702,0) rectangle (axis cs:-4.77118369501691,0.366992928197243);
\draw[draw=none,fill=chocolate2267451,fill opacity=0.5, very thin] (axis cs:-4.77118369501691,0) rectangle (axis cs:-4.29356580516363,0.500881573836011);
\draw[draw=none,fill=chocolate2267451,fill opacity=0.5, very thin] (axis cs:-4.29356580516363,0) rectangle (axis cs:-3.81594791531035,0.439218647435502);
\draw[draw=none,fill=chocolate2267451,fill opacity=0.5, very thin] (axis cs:-3.81594791531035,0) rectangle (axis cs:-3.33833002545707,0.261265900769938);
\draw[draw=none,fill=chocolate2267451,fill opacity=0.5, very thin] (axis cs:-3.33833002545707,0) rectangle (axis cs:-2.86071213560379,0.0995537395602092);
\draw[draw=none,fill=chocolate2267451,fill opacity=0.5, very thin] (axis cs:-2.86071213560379,0) rectangle (axis cs:-2.3830942457505,0.0178378289644126);
\draw[draw=none,fill=chocolate2267451,fill opacity=0.5, very thin] (axis cs:-2.3830942457505,0) rectangle (axis cs:-1.90547635589722,0.00154477930574162);
\draw[draw=none,fill=chocolate2267451,fill opacity=0.5, very thin] (axis cs:-1.90547635589722,0) rectangle (axis cs:-1.42785846604394,5.24642028365079e-05);
\addplot [line width=1.5pt, chocolate2267451]
table {%
-10.9802162631096 1.8324560061517e-05
-10.883727800513 1.91434329640196e-05
-10.7872393379163 1.40062116554309e-05
-10.6907508753197 1.36864211087045e-05
-10.5942624127231 1.33269522557325e-05
-10.4977739501265 3.12268707429163e-06
-10.4012854875298 1.6950065629606e-05
-10.3047970249332 4.96281959814513e-05
-10.2083085623366 5.56395881648941e-05
-10.11182009974 4.05933796666948e-05
-10.0153316371434 2.38737977661094e-05
-9.91884317454673 6.36252476131153e-05
-9.82235471195011 6.65058279036173e-05
-9.72586624935349 8.03838634420449e-05
-9.62937778675686 0.000107949450057334
-9.53288932416024 0.00011114126940553
-9.43640086156362 0.000197057505915192
-9.339912398967 0.000275248312584977
-9.24342393637037 0.000334996524268612
-9.14693547377375 0.00033843504369304
-9.05044701117713 0.000374225909344772
-8.95395854858051 0.00041602068329518
-8.85747008598388 0.000525722425076681
-8.76098162338726 0.000696529068731383
-8.66449316079064 0.000885881958868526
-8.56800469819401 0.00107296614113949
-8.47151623559739 0.00131584531756609
-8.37502777300077 0.0015944477283189
-8.27853931040415 0.00182566762957087
-8.18205084780752 0.00221189520065659
-8.0855623852109 0.00249475889763639
-7.98907392261428 0.00326047801117759
-7.89258546001766 0.00418845535649521
-7.79609699742103 0.00509548176176085
-7.69960853482441 0.0060928659620836
-7.60312007222779 0.0074770335034292
-7.50663160963117 0.00907828888768389
-7.41014314703454 0.0109908674413151
-7.31365468443792 0.0130421664889862
-7.2171662218413 0.0148806938779518
-7.12067775924467 0.0180295682327471
-7.02418929664805 0.0217069797686581
-6.92770083405143 0.0247615490277235
-6.83121237145481 0.0280194095526063
-6.73472390885818 0.0323748087091113
-6.63823544626156 0.0379083001666801
-6.54174698366494 0.044185929315759
-6.44525852106832 0.0503746207927306
-6.34877005847169 0.0592427901849369
-6.25228159587507 0.0685046942772336
-6.15579313327845 0.0795643754761353
-6.05930467068183 0.0915691369842665
-5.9628162080852 0.104827263432566
-5.86632774548858 0.120669250637444
-5.76983928289196 0.136678494520455
-5.67335082029533 0.155209605153672
-5.57686235769871 0.177901094113302
-5.48037389510209 0.203552069646134
-5.38388543250547 0.234214520125348
-5.28739696990884 0.267373787295538
-5.19090850731222 0.302847082413863
-5.0944200447156 0.337413687602391
-4.99793158211898 0.37128133551518
-4.90144311952235 0.405931411963295
-4.80495465692573 0.440778289992176
-4.70846619432911 0.473560267852145
-4.61197773173249 0.494809640333886
-4.51548926913586 0.51042971848947
-4.41900080653924 0.515353087265356
-4.32251234394262 0.507554229566843
-4.22602388134599 0.488476855698796
-4.12953541874937 0.464109873305903
-4.03304695615275 0.435463185190283
-3.93655849355613 0.402607452874398
-3.8400700309595 0.364819732079173
-3.74358156836288 0.324382799198214
-3.64709310576626 0.287687438219869
-3.55060464316964 0.251902150426121
-3.45411618057301 0.214429604940199
-3.35762771797639 0.17870371882681
-3.26113925537977 0.145992372648035
-3.16465079278315 0.11468326593089
-3.06816233018652 0.0877971920111891
-2.9716738675899 0.0655260584609568
-2.87518540499328 0.0464238686380629
-2.77869694239665 0.032348080973793
-2.68220847980003 0.0203786788475584
-2.58572001720341 0.0125133273078373
-2.48923155460679 0.0080953688010043
-2.39274309201016 0.00503616447039177
-2.29625462941354 0.003088390978216
-2.19976616681692 0.00177148475452993
-2.1032777042203 0.000879026402552121
-2.00678924162367 0.000412654836767162
-1.91030077902705 0.000209240851858436
-1.81381231643043 0.000101268132645136
-1.71732385383381 3.2798274840902e-05
-1.62083539123718 7.59409804773925e-06
-1.52434692864056 1.30110659740777e-05
-1.42785846604394 3.29324116424712e-05
};
\end{axis}

\end{tikzpicture}

%% file: arXiv_supp.tex
\section*{Overview}
\label{sec:over}

 In this Supplementary Material, we prove more bounds related to geometric errors
 \begin{align} \label{eq:geometric_err SM}
  \mathcal{E}_{G}^2(\zz,\vec{\theta}) =\min_{\ee}\quad & \norm{\ee}^2 \\
  \text{s.t. }\quad & \CC(\zz + \ee, \vec{\theta}) = 0. \label{eq:geometric_err_C SM}
\end{align}
In \Cref{s: low b}, we generalize \Cref{prop: eS lower} to the most general setting with $N$ constraints of any degree. In \Cref{s: complete int}, we give a result in the spirit of \Cref{prop: eG bound} for $N$ quadric constraints. We restrict to quadric polynomials in the case of multiple constraints (this applies to the epipolar constraints). 

In \Cref{sec:opt} we provide details regarding the optimization of Sampson approximations, and in \Cref{sec:extra_experiment} we show additional results from the experiments in the main paper.

%%%%%%%%%%%%%%%%%%%%%%%%%%%%%%%%%%%%%%%%%%%%%%%%%%%%%%%%%%%%%%%%%%%%%%%%%%%%%%%%%%%%

\section{General Case Lower Bound for $\|\ee^G\|$}\label{s: low b}

In our pursuit to understand constraints of any degrees, we make use of $d$-norms:
\begin{align}
    \|\xx\|_d:=\sqrt[d]{x_1^d+\cdots +x_n^d}.
\end{align}
The $2$-norm $\|\xx\|_2$ is also simply denoted by $\|\xx\|$. For matrices however, $\|\cdot\|$ will refer to the operator norm. We apply the following estimation of polynomials of general degrees.

\begin{lemma}\label{le: q} Let $q:\mathbb R^n\to \mathbb R$ be a homogeneous polynomial in $n$ variables of degree $d$. Then
\begin{align}\label{eq: estim}
    |q(\xx)|\le\mu_q\|\xx\|_d^d, 
\end{align}
where $\mu_q$ is the sum of absolute values of coefficients of $q$.
\end{lemma}

This lemma can be extended to non-homogeneous polynomials in a straightforward way, by replacing $\xx$ by $(\xx;1)$ in \eqref{eq: estim}.

\begin{proof} Our first observation is that
\begin{align}
    |q( \xx)|\le  \big(\max_{\|\yy\|_d=1}|q(\yy)|\big)\|\xx\|_d^d. 
\end{align}
For $\|\yy\|_d=1$ we can bound $|q(\yy)|$ from above by the sum of the absolute values of the coefficients of $q$ because, for $x$ with $\|x\|_d=1$, each coordinate $x_i$ has norm $\le 1$ and therefore the norms of monomials are bounded by $1$.
\end{proof}

We extend the notation from the main body of the paper. A set of polynomial constraints  \begin{equation}
    \CC(\zz) = \left(C_1(\zz), C_2(\zz),~\dots,C_N(\zz)\right)^\tp = \vec{0},
\end{equation}
for $\zz\in \mathbb R^n$, can be expressed via a Taylor expansion:
\begin{align}
    C_i(\zz+\ee)=\CC(\zz)+\sum_{j=1}^d \frac{1}{j!} \ee \times \mathcal T_j^{(i)}, 
\end{align}
for each $i$, where $\mathcal T_j^{(i)}$ is a symmetric $n\times \cdots \times n$ tensor of order $j$, and 
\begin{align}\label{eq: tensor times}
    \ee \times \mathcal T_j^{(i)}=\sum_{l_1,\ldots,l_j\in [n]} (\mathcal T_j^{(i)})_{l_1,\ldots,l_j}\varepsilon_{l_1}\cdots \varepsilon_{l_j}.
\end{align}
For example, $\mathcal T_1^{(i)}$ is the Jacobian of $C_i$ and $\mathcal T_2^{(i)}$ is the Hessian of $C_i$.  

As in \eqref{eq: tensor times}, each tensor $\mathcal T_j^{(i)}$ defines a polynomial, and by \Cref{le: q}, 
\begin{align}\label{eq: mu ij}
    |\ee\times \mathcal T_j^{(i)}|\le \mu_{j}^{(i)}\|\ee\|_j^j,
\end{align}
where $\mu_j^{(i)}$ is the sum of absolute values of coefficients of this polynomial.  Let $\ee\times \mathcal T_j$ denote the vector of $N$ coordinates $(\ee\times \mathcal T_j^{(i)})_{i=1}^N$. Write $\mu_j$ for the sum of all $\mu_j^{(i)}$ for $i=1,\ldots,N$, and note that
\begin{align}\label{eq: mu ij full}
    \|\ee\times \mathcal T_j\|\le \mu_{j}\|\ee\|_j^j.
\end{align}

Recall from \Cref{ss: rank-def} that putting $\Sig=I$ gives us
\begin{align}
    \ee^S=-\JJ^\dagger \CC(\zz).
\end{align}

\begin{proposition} Assume that the optimization problem \eqref{eq:geometric_err SM} has $N$ homogeneous constraints of at most degree $d$, that $\JJ$ evaluated at $\zz$ has linearly independent rows and that $\CC(\zz)\in \mathrm{Im}\;\JJ$. Then
    \begin{align}
      \|\ee^S\|\le \|\ee^G\|+\|\JJ^\dagger\|\sum_{j=2}^d \frac{\mu_j}{j!}\|\ee^G\|_j^j.
    \end{align}
\end{proposition}

Recall that $\|\JJ^\dagger\|$ refers to the operator norm of the pseudo-inverse of $\JJ$.

\begin{proof} For $\ee^G$, we have 
\begin{align}
   0&= \CC(\zz)+\JJ\ee^G+\sum_{j=2}^d \frac{1}{j!} \ee^G\times \mathcal T_j\\
   &=\CC(\zz)+\JJ\left(\ee^G+\JJ^\dagger \sum_{j=2}^d \frac{1}{j!} \ee^G\times \mathcal T_j\right),
\end{align}
meaning that the norm of $\ee^S$ must be bounded from above by the norm of 
\begin{align}
    \ee^G+\JJ^\dagger \sum_{j=2}^d \frac{1}{j!} \ee^G\times \mathcal T_j.
\end{align}
However, by \eqref{eq: mu ij full} the statement now follows.
\end{proof}

\section{Multiple Quadric Constraints}\label{s: complete int}

In this section we prove an upper bound for the geometric error in the case of multiple constraints, and take a closer look at this bound in an example with two quadric constraints. Our main tool is the following celebrated result, as stated in \cite[Ch. 2, Cor. 2.15]{hatcher2005algebraic}.

\begin{theorem}[Brouwer's Fixed Point Theorem] Every continuous function $f$ from a non-empty convex compact subset $K$ of a Euclidean space to $K$ itself has a fixed point.
\end{theorem} 
Recall that a fixed point of a function is a point $x^*$ in the set $K$ such that $f(x^*)=x^*.$   
The main theorem of this section deals with varieties $X$ that are a complete intersection, defined by $N$ quadrics $\CC(\zz)=(C_1(\zz),\ldots,C_N(\zz))^\tp$. By complete intersection, we mean that the dimension of $X$ is $n-N$. We discuss the general case afterwards.   

%We first need more notation. 
Define
\begin{align}
    \sigma_i&:= \textnormal{ spectral radius of } \frac{\|\JJ\|}{2}(\JJ^\dagger)^\tp\vec{H}_i\JJ^\dagger,\\
    \mathrm{cond}(\JJ)&:=\|\JJ\|\|\JJ^\dagger\|.
\end{align}
Neither $\sigma_i$ nor $\mathrm{cond}(\JJ)$ are independent of the individual scalings of $C_i$ for each $i=1,\ldots,N$. They are however independent of simultaneous scaling of all constraints.

Since we will deal with multiple constraints of degree two, we first recall a classical fact about the solutions of degree 2 equations in one variable. This will be used in the proof of the main theorem. 

\begin{remark}\label{re: cont sol} Consider the equation
\begin{align}\label{eq: deg two}
    \alpha x^2+\beta x + \gamma=0.
\end{align}
%There is a continuous function $x(\alpha,\beta,\gamma)$ that outputs a real solution to \eqref{eq: deg two} for inputs in the semi-algebraic set $\beta^2-4\alpha\gamma\ge 0,\alpha \neq 0$. To see this, note that as long as $\alpha\neq 0$, the solutions to the equation are
As long as $\alpha \neq 0$, the solutions to the equation are
\begin{align}
    x=\frac{-\beta\pm \sqrt{\beta^2-4\alpha\gamma}}{2\alpha}.
\end{align}
We can consider these solutions as a function of $\alpha, \beta$ and $\gamma$ that outputs a real solution to \eqref{eq: deg two} for inputs in the semi-algebraic set $\beta^2-4\alpha\gamma\ge 0,\alpha \neq 0$. Moreover, this function is continuous because the square root is continuous. 
%and the square-root is continuous for inputs $\ge 0$. 
The expression $\beta^2-4\alpha\gamma$ is called the discriminant of \eqref{eq: deg two}.
\end{remark}

\begin{theorem}\label{thm: main} Consider a complete intersection defined by the quadratic equations 
\begin{equation}
    \CC(\zz) = \left(C_1(\zz), C_2(\zz),~\dots,C_N(\zz)\right)^\tp = \vec{0}.
\end{equation}
Assume that $\JJ$ is full-rank at $\zz$. If a number $\kappa\ge 0$ satisfies   
    \begin{align}\label{eq: main 1}
       \kappa&\ge  \sum_{j=1}^N \left(\frac{|C_j(\zz)|}{\|\JJ\|}+\sigma_j\kappa\right)^2,\end{align} then 
    \begin{align}
        \|\ee^G\|\le \mathrm{cond}(\JJ)\kappa.
    \end{align}
\end{theorem}

It is easy to check if there is a $\kappa$ such that the conditions \eqref{eq: main 1} holds. Indeed, one solves the quadratic equation  
 \begin{align}
       \kappa&=  \sum_{j=1}^N \left(\frac{|C_j(\zz)|}{\|\JJ\|}+\sigma_j\kappa\right)^2.
\end{align} 
If there exists real solutions, then they are $\ge 0$, because the right-hand side of the equation is always non-negative. In this case, the smallest real solution is the smallest $\kappa$ satisfying \eqref{eq: main 1}.

In order to relate theorem to the Sampson error as described in \Cref{ss: mult const}, we note that if the conditions of the theorem hold for $\kappa=c\|\ee^S\|/\mathrm{cond}(\JJ)$ for some $c$, then $\|\ee^G\|\le c\|\ee^S\|$.

\begin{proof}  Define 
\begin{align}
\ee(\vec{\lambda}):=\|\JJ\|\JJ^\dagger \vec{\lambda}. 
\end{align} 
Then $f_i(\vec{\lambda}):=C_i(\zz+\ee(\vec{\lambda}))/\|\JJ\|$ equals
\begin{align}\begin{aligned}
 \frac{C_i(\zz)}{\|\JJ\|}+\lambda_i+\frac{\|\JJ\|}{2}\vec{\lambda}^\tp(\JJ^\dagger)^\tp \vec{H}_i\JJ^\dagger\vec{\lambda}.
\end{aligned}
\end{align}
We estimate $f_i$ from above and below using $\sigma_i$:

\begin{align}\label{eq: est}
   \frac{C_i(\zz)}{\|\JJ\|}+\lambda_i-\sigma_i\|\vec{\lambda}\|^2\le f_i(\vec{\lambda})\le \frac{C_i(\zz)}{\|\JJ\|}+\lambda_i+\sigma_i\|\vec{\lambda}\|^2.
\end{align}
%We consider the following region: $ \vl\in \mathbb R^N $ such that
This estimation can be refined for $\vl$ in a specific region. Indeed, let $ \vl\in \mathbb R^N $ satisfy 
\begin{align}\label{eq: reg}
  \left|\lambda_i+\frac{C_i(\zz)}{\|\JJ\|}\right|\le \sigma_i\kappa,
\end{align}
for each $i=1,\ldots,N$ and let $\kappa$ be as in the statement. Using the reverse triangle inequality, we have
\begin{align}
|\lambda_i|\le   \frac{|C_i|}{\|\JJ\|}+\sigma_i\kappa,
\end{align}
and it follows by \eqref{eq: main 1} that $\sum_{j=1}^N\lambda_j^2\le \kappa$.

Then from \eqref{eq: est}, we get
\begin{align}\label{eq: est prim}
   \frac{C_i(\zz)}{\|\JJ\|}+\lambda_i-\sigma_i\kappa\le f_i(\vec{\lambda})\le \frac{C_i(\zz)}{\|\JJ\|}+\lambda_i+\sigma_i\kappa,
\end{align}
and this estimation only depends on $\lambda_i$. Fixing each $\lambda_j,j\neq i$ in \eqref{eq: reg}, the solutions $\lambda_i^{\pm}$ to the two linear equations are 
\begin{align}
\lambda_i^{\pm}=-\frac{C_i(\zz)}{\|\JJ\|} \mp \sigma_i\kappa,
\end{align}
which satisfy \eqref{eq: reg}. Note that by \eqref{eq: est prim}, for $\lambda_i^-$, $f_i(\vl)\ge 0$ and for $\lambda_i^+$, $f_i(\vl)\le 0$. It follows that there must be a real $\lambda_i^*$ such that $f_i(\vl)=0$ in the interval 
\begin{align}
    \left|\lambda_i^*+\frac{C_i(\zz)}{\|\JJ\|}\right|\le \sigma_i\kappa,
\end{align}
because this interval contains both $\lambda_i^{\pm}$.

%For $\vec{t}\in \mathbb R^N$ let $\hat{\vec{t}}_i\in \mathbb R^{N-1}$ be the vector we get by removing the $i$-th coordinate. 
The existence of a real solution implies that the discriminant is greater or equal than $0$ in \eqref{eq: reg} and by \Cref{re: cont sol} we have that $\lambda_i^*(\hat{\vl}_i)$ are continuous functions (here $\hat{\vl}_i$ denotes the vector in $\reals^{N-1}$ obtained by removing the $i$-th coordinate from $\vl\in\reals^N$).

Let $K$ be the hypercube defined in \eqref{eq: reg}. In order to apply Brouwer's fixed point theorem we consider the continuous function
\begin{align}
F: K&\to K,\\
    \vl &\to (\lambda_1^*(\hat{\vl}_1),\ldots,\lambda_N^*(\hat{\vl}_N)).
\end{align}
By the theorem, there is a fixed point $\vl^*\in K$ with the property that $F(\vl^*)=\vl^*$. This means exactly that $(\vl^*)_i=\lambda_i^*(\hat{\vl}_i)$ for each $i=1,\ldots,N$. By construction, $\lambda_i^*(\hat{\vl}_i)$ solves $f_i(\vl)=0$ for fixed $\lambda_j,j\neq i$. This means that $f_i(\vl^*)=0$ for each $i=1,\ldots,N$.

To summarize, there exists a $\vl^*\in K$ such that $C_i(\zz+\ee(\vl^*))=0$ for each $i=1,\ldots,N$. This means that $\|\ee^G\|\le \|\ee(\vl^*)\|$. Further, as we have defined $\ee$, $\|\ee(\vl^*)\|\le \|\JJ\|\|\JJ^\dagger\|\|\vl^*\|^2$. Finally, $\|\vl^*\|^2\le \kappa$ as noted above and we are done. 
\end{proof}

In the general case, where we are given $N$ quadric constraints that define a variety $X$ of dimension $m$ that is not necessarily a complete intersection, we can use the fact that locally, it is defined by $n-m$ constraints. To be precise, the Jacobian at a generic point $\xx$ of $X$ is of rank $n-m$, and any choice of $n-m$ constraints with full-rank Jacobian locally describe $X$ around $\xx$. Heuristically, given a data point $\zz$ outside the variety, we choose $n-m$ constraints for which the Jacobian has full-rank and apply \Cref{thm: main} to these constraints. We leave it to future work to make this rigorous.% In order to make this rigorous, one could prove (under additional assumptions) that for the solution $\vl^*$ constructed in the proof of \Cref{thm: main}, the Jacobian at the point $\zz+\ee(\vl^*)$ is full-rank. We leave such matters for future work.   

We illustrate the theorem with an example below.

\begin{example} Consider the two quadratic constraints defining a variety in $\mathbb R^3$,
\begin{align}
    x^2+y^2+z^2-1&=0,\\
    z-xy&=0.
\end{align}
This curve is a complete intersection, and the associated Jacobian is 
\begin{align}
    \JJ= & \begin{bmatrix}
        2x & 2y & 2z \\
        -y & -x & 1
    \end{bmatrix}.
    % \JJ_1&=\begin{bmatrix}
    %     2x & 2y & 2z
    % \end{bmatrix}, \vec{H}_1&=\begin{bmatrix}
    %     2 & 0 & 0\\ 0 & 2 & 0\\ 0 & 0 &2
    % \end{bmatrix}\\
    %  \JJ_2&=\begin{bmatrix}
    %     -y & -x & 1
    % \end{bmatrix}, \vec{H}_2&=\begin{bmatrix}
    %     0 & -1 & 0\\ -1 & 0 & 0\\ 0 & 0 &0
    % \end{bmatrix}.
\end{align}
%We have $\JJ_1^\dagger=\JJ_1^\tp /4(x^2+y^2+z^2)$ and $\JJ_2^\dagger=\JJ_2^\tp /(x^2+y^2+1)$.
%One can check that $x=y=\sqrt{\sqrt{2}-1},z=\sqrt{2}-1$ is a solution. Generically, $\JJ$ is rank-two and the pseudo-inverse takes the form
%that generically has rank 2. 
% \begin{align}
%     \JJ^\dagger=\frac{1}{A}\begin{bmatrix}
%         2x & -y\\ 2y & -x \\ 2z & 1
%     \end{bmatrix}\begin{bmatrix}
%         x^2+y^2+1 & -2z\\ -z & 4(x^2+y^2+z^2)
%     \end{bmatrix}.
% \end{align}
% with $A=4(x^2+y^2+z^2)(x^2+y^2+1)-2z^2.$

The bound coming from \Cref{thm: main} can be used as long as there exists a $k$ such that 
\begin{align}
   \kappa&\ge  \sum_{j=1}^N \left(\frac{|C_j(\zz)|}{\|\JJ\|}+\sigma_j\kappa\right)^2.
\end{align}

Note, that this inequality will have a real solution depending on the value of $\zz.$ To illustrate how often this bound is satisfied, we conducted the following numerical experiment:
\begin{itemize}
    \item We sample $m$ data points in the curve,
    \item introduce an error $\epsilon$ in each point and generate a noisy sample of size $m$, 
    \item for each noisy point we compute $C_i, \|J\|$ and $\sigma_i$ and decide whether the inequality \eqref{eq: main 1} has a solution. 
    \item We count the percentage of points in the sample that had a positive result in the previous step. 
\end{itemize}
We present our results in \Cref{fig:hist}

\begin{figure}
    \centering
    \begin{subfigure}[h]{0.2\textwidth}
        \centering
        \includegraphics[scale=0.28]{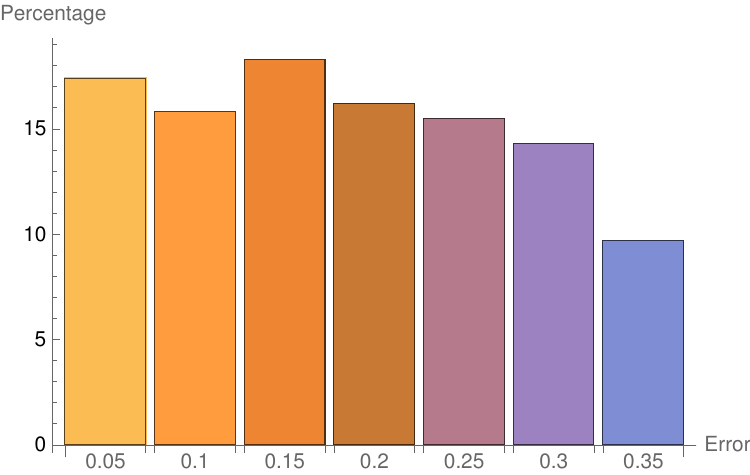}
        \caption{$m=200$}  
    \end{subfigure}
    \hfill
    \begin{subfigure}[h]{0.2\textwidth}
        \centering
        \includegraphics[scale=0.28]{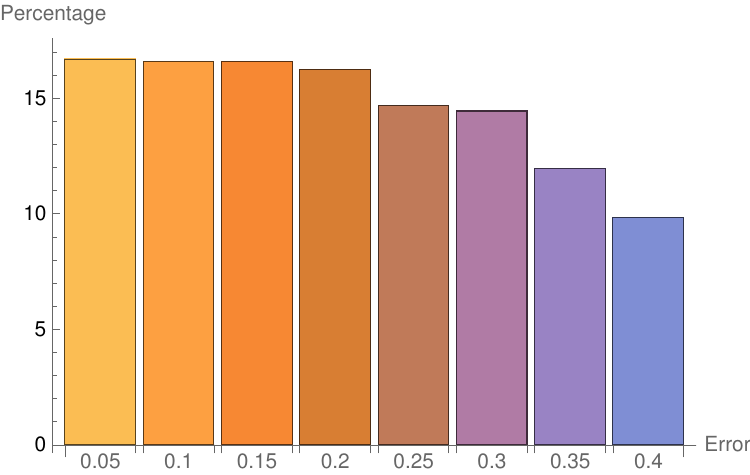}
        \caption{$m=500$}  
    \end{subfigure}
    \begin{subfigure}[h]{0.2\textwidth}
        \centering
        \includegraphics[scale=0.28]{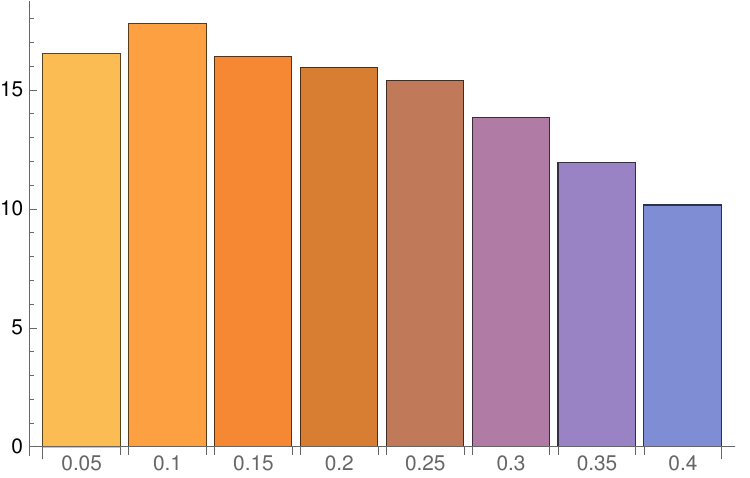}
        \caption{$m=1000$}  
    \end{subfigure}
    \caption{Percentage of data points whose geometric error can be bounded using \Cref{thm: main}. These percentages were computed over noisy samples of 200, 500 and 1000 points, and are depicted respectively in the histograms.  }
    \label{fig:hist}
\end{figure}

%\FR{Try to plot \Cref{eq: sig} and study numerically small perturbation around a point on the variety.}

%\FR{Plot numerically what percentage satisfies the constraints at different noise levels.}
\end{example}

%%%%%%%%%%%%%%%%%%%%%%%%%%%%%%%%%%%%%%%%%%%%%%%%%%%%%%%%%%%%%%%%%%%%%%%%%%%%%%%%%%%%%%%%%%%%%%%%%%%%%%%%%%%%%%%%%%%%%%%%%%%%%%%%%%%%%%%%%%%%%%%%%%%

\iffalse
\section{Updated Linearization Point}
\FR{Do we keep this section? }

In the original Sampson error we linearize at the measurement. However we can also consider linearization around another point.

\begin{equation}
    \min_{\xx} \| \xx - \xx_0 \|^2_\Sig \quad \text{s.t.} \quad \CC(\xx_k) + J (\xx - \xx_k) = 0
\end{equation}
where $J$ is computed at $\xx_k$. We can rewrite in terms of $\ee$,
\begin{equation}
    \min_{\ee} \| \ee \|^2_\Sig \quad \text{s.t.} \quad \CC(\xx_k) + J (\ee - \ee_k) = 0
\end{equation}
where $\ee_k = \xx_k - \xx_0$. 
Using the same derivations as before we get
\begin{equation}
     \| \ee \|^2_\Sig = \| (J\Sig^{1/2})^\dagger (\CC(\xx_k) - J\ee_k) \|^2
\end{equation}
\fi

\section{Optimization of Sampson Approximations} \label{sec:opt}
The constraint $C(\zz,\theta)$ typically depends not only on the measurements $\zz$, but also some model parameters $\theta$ which we are estimating. Fitting the parameters we want to minimize the residuals for each measurement $\zz_k, ~k=1,\dots,m$.

\begin{equation}
    \theta^\star = \arg\min_\theta \sum_k \| \JJ_k^\dagger \CC(\zz_k, \theta) \|^2
\end{equation}
where $\JJ_k = \frac{\partial \CC(\zz,\theta)}{\partial \zz}\rvert_{\zz=\zz_k}$.

To apply standard non-linear least squares algorithms (e.g.~Levenberg-Marquardt), we need to evaluate the Jacobian of the residuals $\vec{r}_k = \JJ_k^\dagger \CC(\zz_k)$ w.r.t.~$\theta$, i.e.
\begin{align}
     \frac{\partial \vec{r}_k}{\partial \theta} &= \frac{\partial}{\partial \theta} \left( (\frac{\partial \CC}{\partial \zz} )^\dagger \CC(\zz,\theta) \right) \\
     &= \left( \frac{\partial}{\partial \theta} (\frac{\partial \CC}{\partial \zz})^\dagger \right) \CC(\zz,\theta) +  (\frac{\partial \CC}{\partial \zz})^\dagger \frac{\partial \CC}{\partial \theta}
\end{align}
Denote $\JJ_{\zz} = \frac{\partial \CC}{\partial \zz}$ and $\JJ_\theta = \frac{\partial \CC}{\partial \theta}$, then
\begin{align}
     \frac{\partial}{\partial \theta} \JJ_{\zz}^\dagger =
      -\JJ_{\zz}^\dagger \JJ_\theta \JJ_{\zz}^\dagger
      + \JJ_{\zz}^\dagger (\JJ_{\zz}^\dagger)^T \JJ_\theta (I - \JJ_{\zz} \JJ_{\zz}^\dagger)
\end{align}
See Golub and Pereyra \cite{golub1973differentiation} for more details.

%%%%%%%%%%%%%%%%%%%%%%%%%%%%%%%%%%%%%%%%%%%%%%%%%%%%%%%%%%%%%%%%%%%%%%%%%%%%%%%%%%%%%%%%%%%%%%%%%%%%%%%%%%%%%%%%%%%%%%%%%%%%%%%%%%%%%%%%%%%%%%%%%%%%%

\section{Additional Experimental Results}  \label{sec:extra_experiment}
In Table~\ref{tbl:7scenes_full_results} we show the full per-scene results of the covariance aware camera pose refinement from Section 4.4~in the main paper, and Table~\ref{tbl:ref_vp} show the results of refining the vanishing points using both the mid-point error and the Sampson error. 
For this data there was no difference in the results of the refinement.

\begin{table}[ht]
\centering
\begin{tabular}{l c c l c c} \toprule
& \multicolumn{2}{c}{YUB+\cite{yorkurban}} && \multicolumn{2}{c}{NYU \cite{silberman_2012}\cite{kluger2020consac}} \\ \cmidrule{2-3} \cmidrule{5-6}
& Mean  & AUC && Mean  & AUC \\ \midrule
VP from \cite{pautrat2023deeplsd} & 1.62  &  0.86 && 3.24 & 0.70 \\
~\rotatebox[origin=c]{180}{$\Lsh$} Mid-point & 1.57 & 0.86 && 3.24 & 0.70 \\
~\rotatebox[origin=c]{180}{$\Lsh$} Sampson. & 1.57 & 0.86 &&3.24 & 0.70 \\ \bottomrule
\end{tabular}
\caption{\textbf{Refinement of vanishing points.}}
\label{tbl:ref_vp}
\end{table}

\begin{table*}[ht]
    \centering
    \resizebox{\textwidth}{!}{
    \begin{tabular}{l l c  c  c  c  c  c  c  ccc} \toprule
         %& \multicolumn{2}{c}{Chess} & \multicolumn{2}{c}{Fire} & \multicolumn{2}{c}{Heads} & \multicolumn{2}{c}{Office} & \multicolumn{2}{c}{Pumpkin} & \multicolumn{2}{c}{Redkitchen} & \multicolumn{2}{c}{Stairs} & \multicolumn{3}{c}{Average} \\
         $\tau$ && Chess &  Fire & Heads & Office & Pumpkin & Redkitchen & Stairs & Average \\         \midrule
\multirow{3}{*}{\rotatebox[origin=c]{90}{$5$px}}
&Reproj.     & 0.86 / 2.47 & 0.84 / 2.11 & 0.75 / 1.07 & 0.89 / 3.05 & 1.24 / 4.78 & 1.39 / 4.15 & 1.22 / 4.44 & 1.03 / 3.20\\ 
&Reproj+Cov  &0.85 / 2.45 & 0.82 / 2.04 & 0.74 / 1.02 & 0.87 / 2.99 & 1.22 / 4.75 & 1.36 / 4.03 & 1.15 / 4.28 & 1.01 / 3.10 \\
&~\rotatebox[origin=c]{180}{$\Lsh$} Sampson     &  0.85 / 2.45 & 0.82 / 2.04 & 0.74 / 1.02 & 0.87 / 2.99 & 1.22 / 4.74 & 1.36 / 4.02 & 1.14 / 4.27 & 1.01 / 3.10  \\ \midrule
\multirow{3}{*}{\rotatebox[origin=c]{90}{$10$px}}
&Reproj.     &0.84 / 2.42 & 0.90 / 2.25 & 0.82 / 1.18 & 0.92 / 3.07 & 1.25 / 4.79 & 1.39 / 4.20 & 1.32 / 4.78 & 1.06 / 3.24 \\
&Reproj+Cov  &0.79 / 2.38 & 0.81 / 2.03 & 0.73 / 1.03 & 0.86 / 2.92 & 1.20 / 4.41 & 1.32 / 3.83 & 1.12 / 4.17 & 0.99 / 2.99  \\
&~\rotatebox[origin=c]{180}{$\Lsh$} Sampson     &0.79 / 2.37 & 0.81 / 2.03 & 0.73 / 1.03 & 0.86 / 2.94 & 1.20 / 4.40 & 1.32 / 3.85 & 1.12 / 4.20 & 0.98 / 2.99  \\ \midrule
\multirow{3}{*}{\rotatebox[origin=c]{90}{$20$px}}
&Reproj.     &0.87 / 2.53 & 1.08 / 2.72 & 1.04 / 1.45 & 1.06 / 3.43 & 1.38 / 5.41 & 1.49 / 4.50 & 1.99 / 6.84 & 1.21 / 3.63  \\
&Reproj+Cov  &0.75 / 2.30 & 0.80 / 2.06 & 0.73 / 1.03 & 0.88 / 2.91 & 1.13 / 4.30 & 1.29 / 3.81 & 1.36 / 4.89 & 0.99 / 3.00 \\
&~\rotatebox[origin=c]{180}{$\Lsh$} Sampson     &0.75 / 2.31 & 0.80 / 2.07 & 0.73 / 1.03 & 0.88 / 2.92 & 1.13 / 4.34 & 1.29 / 3.88 & 1.44 / 5.25 & 1.00 / 3.02  \\ \bottomrule
    \end{tabular}}
    \caption{\textbf{Full results for 7Scenes}. Table shows the median  rotation (deg.)  and translation (cm) errors for each scene in the 7Scenes dataset for the experiment in Section 4.4 in the main paper.}
    \label{tbl:7scenes_full_results}
\end{table*}

%%%%%%%%% REFERENCES
%{\small
%\bibliographystyle{ieee_fullname}
%\bibliography{main}
%}

%\end{document}